\documentclass[%
 reprint,
superscriptaddress,
 amsmath,amssymb,
 aps,
]{revtex4-2}

\usepackage{graphicx}
\usepackage{dcolumn}
\usepackage{bm}
\usepackage{amssymb,amstext,amsmath,amsfonts,amsthm}
\usepackage{xr}
\usepackage{color}
\usepackage{algorithm,algorithmic}
\usepackage{graphicx}
\usepackage{subfigure}
\usepackage{textcomp}
\usepackage{diagbox}
\usepackage{makecell}
\usepackage{multirow}
\usepackage{booktabs}

\definecolor{darkred}{rgb}{0,0,0}
\definecolor{darkblue}{rgb}{0,0,0}

\newcommand{\E}{{\mathbb E}}
\newcommand{\Rbb}{\mathbb R}

\newcommand{\Cbb}{\mathbb C}

\newcommand{\Zbb}{\mathbb Z}
\newcommand{\Tbb}{\mathbb T}

\newcommand{\Bb}{\mathbf B}
\newcommand{\Cb}{\mathbf C}
\newcommand{\Db}{\mathbf D}

\newcommand{\Hb}{\mathbf H}

\newcommand{\Lb}{\mathbf L}

\newcommand{\Qb}{\mathbf Q}

\newcommand{\Ub}{\mathbf U}
\newcommand{\Vb}{\mathbf V}

\newcommand{\Xb}{\mathbf  X}

\newcommand{\df}{{\rm d}}
\newcommand{\pdf}{\partial}
\newcommand{\dd}{{\rm d}^2}

\newcommand{\Acal}{\mathcal A}
\newcommand{\Bcal}{\mathcal B}

\newcommand{\Ncal}{\mathcal N}

\newcommand{\diag}{{\rm diag}}

\newcommand{\Span}{\operatorname{span\,}}

\newcommand{\rank}{\operatorname{rank\,}}

\newcommand{\trace}{\operatorname{tr}}
\newcommand{\vecc}{\operatorname{vec}}
\newcommand{\unvec}{\operatorname{unvec}}

\newtheorem{thm}{Theorem}
\newtheorem{cor}[thm]{Corollary}
\newtheorem{lem}[thm]{Lemma}
\newtheorem{prop}[thm]{Proposition}

\begin{document}

\preprint{APS/123-QED}

\title{ARMAX identification of low rank graphical models}

\author{Wenqi Cao}
\affiliation{%
Center for Systems and Control, College of Engineering, Peking University, Beijing 100871, China}
\author{Aming Li}
\thanks{Corresponding author: amingli@pku.edu.cn}
\affiliation{%
Center for Systems and Control, College of Engineering, Peking University, Beijing 100871, China}
\affiliation{
Center for Multi-Agent Research, Institute for Artificial Intelligence, Peking University, Beijing 100871, China}



\date{\today}

\begin{abstract}
In large-scale systems, complex internal relationships are often present. Such interconnected systems can be effectively described by low rank stochastic processes.
When identifying a predictive model of low rank processes from sampling data, 
the rank-deficient property of spectral densities is often obscured by the inevitable measurement noise in practice.
However, existing low rank identification approaches often did not take noise into explicit consideration, leading to non-negligible inaccuracies even under weak noise.
In this paper, we address the identification issue of low rank processes under measurement noise.
We find that the noisy measurement model admits a sparse plus low rank structure in latent-variable graphical models.
Specifically, we first decompose the problem into a maximum entropy covariance extension problem, and a low rank graphical estimation problem based on an autoregressive moving-average with exogenous input (ARMAX) model.
To identify the ARMAX low rank graphical models, we propose an estimation approach based on maximum likelihood. 
The identifiability and consistency of this approach are proven under certain conditions.
Simulation results confirm the reliable performance of the entire algorithm in both the parameter estimation and noisy data filtering.

\end{abstract}

\maketitle

\section{Introduction}\label{sec:introduction}
Large-scale systems become increasingly prevalent in science and technology nowadays, with the development of system science and artificial intelligence. 
As a typical feature, the high-dimensional property of such systems leads to complex inner relationships among the components within. 
The aim of better describing and understanding large-scale systems with internal correlations has inspired a multitude of fundamental research on graphical models and networks, such as \cite{Dahlhaus00,Zorzi24,Lin13,Pilonetto25,PillonettoPNAS,Li17,WTC19,Augusto21,Augusto18, BGHP17,Chiuso12,WVD18,GraphicalBook10,topo10,Yu18,Mu23,LiuNature}.
The applications of these studies span across diverse fields, including economics and financial analysis \cite{finance16,ecom_social_10}, multi-agents and robotics \cite{multi_agents18,robo10}, ecological evolution \cite{ecology12}, swarm intelligence and sociology \cite{ecom_social_10,social18}, etc.

Denote the nodes in a graph or network by a (possibly high-dimensional) vector process, the existence of a low rank character has been theoretically proven \cite{CLPtac23,CPLauto23,BLM19}, and empirically verified by plentiful data from different real networks \cite{lowrankNP24,sparseNP24}.
In the time domain, the low rank character is revealed by the rank deficiency of a transition matrix in general dynamic network models \cite{lowrankNP24,BLM19}. 
This phenomenon is considered also as the sparsity or singularity of signals and studied in amounts of literature like \cite{Augusto18,EJC10,Chiuso12}.  
The low rank character, along with network representations, allows the system to be driven or determined by only a subset of all nodes. 
This leads to an important topic in the field of complexity science \cite{Li17,LiuNature,ecom_social_10,social18,lowrankNP24,sparseNP24}: identifying the key nodes in networks to efficiently reconstruct or control associated complex networks \cite{Li17,LiuNature,ecom_social_10}.
Initially, key nodes identification problem was studied with known network topologies using static data.
In recent years, it has been expanded to include time-series data with unknown topologies.
Some impressive methodological and theoretical identification studies can be referred to \cite{Yu24SR,Mu17,WTC19}.

For an $(m+l)$-dimensional linear stochastic stationary vector process $y(t)$, the low rank character can be also demonstrated by a rank deficient spectral density matrix $\Phi_y(z)$ of rank $l<m+l$ in the frequency domain. 
In our previous work, the spectra \cite{CL24}, modeling \cite{CLPtac23} and identification \cite{CPLauto23} of low rank vector processes are discussed. 
By rearranging the components of $y(t)$ if necessary, there is a decomposition,
\begin{align}\label{eq:y=ymyl}
    y(t)=\begin{bmatrix}
           y_m(t) \\
           y_l(t)
         \end{bmatrix},
\end{align}
with $y_l$ an $l$-dimensional process with a spectral density $\Phi_{y_l}(z)$ of full rank $l$.
We find a unique deterministic relation always exists in $\Phi_y(z)$ as well as between the two sub-processes of $y(t)$, satisfying 
\begin{align}\nonumber
    y_m(t)=H(z)y_l(t).
\end{align}
This helps understand the internal correlations and reduce the system dimension to the dimension of $y_l(t)$, 
which greatly simplifies the procedures of modeling, identification and sequential kinds of missions. 

The entries in the subprocess $y_l(t)$ can also be regarded as the key nodes of the network composed of all entries in $y(t)$. This fact endows the potential of low rank processes to be utilized in research on graphs and networks, with deterministic relations determined by spectral densities. 
Beyond our work \cite{CPLauto23,CL24}, the identification of such low rank processes has been recently addressed from various important points of view, including dynamic networks \cite{BGHP17,WVD18,Chiuso12,WVD18-2}, singular ARMA models \cite{EJC10,BLM19}, latent state space models \cite{JoeS23}, etc.

Existing exquisite work on identifying low rank or singular processes without input signals include references such as \cite{BGHP17,WVD18,EJC10,BLM19} and our previous work \cite{CPLauto23}, where the low rank process aimed to estimate is motivated by an innovation white noise process in essence. 
However, in practice, signals are generally measured under noise. This ubiquitous scenario has not been explicitly considered in the identification of low rank processes. 
When a low rank process is measured under noise, 
the only accessible measurement process will exhibit a full rank spectral density, even if the intensity of noise is of a small magnitude. 
In this scenario, the hide of low rank processes will result in non-negligible inaccuracies for identification approaches which assume the spectral rank deficiency of sampled time series. 
Therefore, studying the identification of low rank processes given noisy measurements remains a challenging problem.
The issue will be investigated in this paper.  

Considering the description and exploration of the internal correlations in a vector process, graphical models emerge as a reliable tool under measurement noise \cite{Dahlhaus00,Lin13,Zorzi24,Augusto18,GraphicalBook10,topo10,Augusto21}. 
Specifically, the conditional dependencies between pairwise entries are denoted by an edges between nodes in a graphical model. 
Various classical approaches can be used to estimate graphical models of time series \cite{GraphicalBook10}, among which the maximum entropy covariance extension tool is particularly popular and has been deeply studied in a series of works \cite{Lin13,Augusto18,Zorzi24,Augusto21}, for its advantage in estimating the topology. 
In the latest work on ARMA graphical identification \cite{Zorzi24}, cepstral coefficients \cite{Bin24} are innovatively utilized in the problem description, facilitating the joint estimation of AR and MA parts.

To further simplify the graph representation, latent-variable graphical model is proposed for variables correlated through a restricted number of latent variables \cite{Zorzi16}. By introducing latent variables, the inverse of the manifest (observed) spectrum can be decomposed into a sum of a sparse matrix plus a low rank matrix.
The existing several identification approaches for latent-variable graphical models are designed utilizing this sparse plus low rank (S+L) structure  \cite{Zorzi16,Yu23cdc,Yu24auto}. 
In the recent impressive work \cite{Yu24auto}, trace approximation and reweighted nuclear norm minimization are used to jointly estimate the graph topology and parameters of AR graphical latent-variable models, and \cite{Yu23cdc} extends this work to ARMA models.
A relevant model for exploring a smaller number of latent variables is factor analysis model, see such as \cite{EJC10,FAbook10,P23,Lucia24}. It is worth mentioning that in \cite{P23}, state space dynamic generalized factor analysis models are proposed for estimation insightfully.
Though latent-variable graphical models or factor models allow a prominent reduction of edges, the edges between latent variables and observed variables are less explainable than original graphical models, because the latent variables lack practical meanings in general.

In this paper, based on the deterministic relation within a rank-deficient spectral density, we find the mathematical representation in identifying a low rank process under measurement noise exhibits a similar two-layer structure like latent-variable graphical models, which we name by low rank graphical model.  
Different from existing latent-variable graphical models, edges in a low rank graphical model are constructed between the measurement process and the full rank subprocess $y_l(t)$.
We shall prove in the paper that the S+L character holds for low rank graphical models, hence a similar identification framework applies.
Compared with latent-variable graphical models, two advantages of low rank graphical models are the more interpretable edges connected to latent nodes with practical meanings, and the topology determined directly from the deterministic relations.

To solve the identification problem of low rank processes under measurement noise, similarly as in the non-noise case in \cite{CPLauto23}, a two-stage algorithm on estimating low rank graphical models is proposed in this paper. 
In the first stage, an AR innovation model will be identified for the latent variable $y_l(t)$ by maximum entropy covariance extensions (see such as \cite{Georgiou02,Lin99,Picci11}).
In the second stage, the function corresponding to the deterministic relation, $H(z)$ will be estimated. 
Because of the measurement noise, the simple least squares problem in the non-noise case changes into a specific ARMAX graphical estimation problem, which is still lack of identification study to our best knowledge. 
In this paper, we initially solve this nonconvex problem by maximum likelihood ARMAX estimation (\cite{Hannan76,Hannan80,Chen93}). To guarantee the performance of the solution, the identifiability of the specific ARMAX model, and the consistency of the estimates with enough data is proven. The convergent performance is also verified by examples.

In addition, a by-product of low rank graphical identification algorithm is a high-accuracy filter providing estimates of $y(t)$ from its noisy measurements, which could be crucial for tasks that depend on reliable process data.

The introduction and identification of low rank graphical models in this work not only significantly increase the accuracy of low rank identification under noise,
but also compensate for the lack of interpretability in latent-variable graphical models and factor models.
Furthermore, they can also be directly utilized to identify key nodes in networks with unknown topologies, offering additional opportunities for exploration in the field of complexity science and its application areas.

The structure of this paper is as follows. In Section \ref{sec:notation} the notations and preliminaries of this paper are listed. 
Section \ref{sec:problem} introduces and specifies the low rank graphical estimation problem, with a framework of the overall algorithm given. 
In particular, subsection~\ref{subsec:problem_LRGmodel} and~\ref{subsec:problem_structure} discuss the definition, superiority, and the S+L character of low rank graphical models; subsection~\ref{subsec:problem_ARMAX} clarifies the ARMAX graphical estimation problem.
Section~\ref{sec:yl} proposes the first stage of our algorithm on identifying an innovation model for the low rank latent variable $y_l(t)$. 
Section~\ref{sec:ARMAX} proposes the second stage of our algorithm on low rank ARMAX graphical estimation, in which the identifiability, specific algorithm, and the convergence are discussed in subsection~\ref{subsec:ARMAX_identifiability},~\ref{subsec:ARMAX_ML} and ~\ref{subsec:ARMAX_convergence}, respectively.
Simulation examples and a sensitive analysis of the algorithm will be given in Section~\ref{sec:example}.
Finally, section~\ref{sec:conclu} is the conclusion of the paper.

\section{Notations and Preliminaries}\label{sec:notation}
The processes we consider in the whole paper are wide-sense stationary second-moment discrete-time stochastic processes. 
Basically, denote by 
$[M]_{kh}$ the entry on the $k$-th row, $h$-th column of matrix $M$. Denote by $\trace(M)$ the trace of matrix $M$. Denote by $M'$, $M^*$, respectively the transpose, transpose conjugate of matrix $M$,  $M^{-*}:= (M^{-1})^* = (M^{*})^{-1}$.
Denote by $\det M$ or $\det(M)$ the determinate of matrix $M$. 
Denote by $I$ an identity matrix.
Denote by signal $\otimes$ the (right) Kronecker product. 
Denote by matrix $K_{n,m}$ the commutation matrix of dimension $nm\times nm$ \cite[p. 54]{Magnusbook}, satisfying ${K_{n,m}}\text{vec}(M)= \text{vec}(M')$ for matrix $M\in \Rbb^{n\times m}$, where $\vecc(M)=\left[ [M]_{11}, \cdots, [M]_{n1},\cdots, [M]_{1m}, \cdots, [M]_{nm} \right]'$. 
Denote by $\unvec(x)$ the operator of recovering a matrix of given size from vector $x$, 
satisfying $\unvec(\vecc(M))=M$.
Denote by $|S|$ the cardinality of set $S$.
Define $z$ as the time-shift operator, satisfying $y(t+1)= zy(t)$ for process $y(t)$. 
Denote by $y_{(k)}(t)$ the $k$-th scalar subprocess of a vector process $y(t)$.
Denote by $N(\mu, \Gamma)$ a Gaussian distribution with expectation $\mu$ and variance $\Gamma$.

{\bf For vector spaces.} Denote by $\Qb_m$ the vector space of symmetric matrices of dimension $m$. For $X \in \Qb_m$, we denote by $X \succ 0$ (or $X \succeq 0$) if $X$ is positive definite (or semi-definite). Denote by $\Db_m$ the vector space of diagonal matrices of dimension $m$.
Denote by $\Vert \cdot \Vert_2$ the Euclidean norm of a vector or a matrix.

{\bf For measurement sequences.} Denote by $R_j := \E [\zeta(t+j)\zeta(t)']$ the $j$-th covariance lag of process $\zeta(t)$, then  given  $\{\zeta(1), \cdots, \zeta(N)\}$, an empirical estimate $\hat{R}_j$ can be computed as
\begin{align}\label{eq:hatRj}
  \hat{R}_j = \frac{1}{N-j} \sum_{t=1}^{N-j} \zeta(t+j)\zeta(t)'.
\end{align}
Denote a finite list of covariance lags by 
\begin{align}\nonumber
    \mathcal{R}:=\begin{bmatrix}R_0 & R_1 & \cdots & R_n \end{bmatrix},\\ \nonumber
    \hat{\mathcal{R}} := \begin{bmatrix}\hat{R}_0 & \hat{R}_1 &  \cdots & \hat{R}_n \end{bmatrix},
\end{align}
except as otherwise stated.

Given the covariance lags $R_0, R_1, \cdots, R_n$ in $\Rbb^{m\times m}$, define the matrix pseudo-polynomial
\begin{align}
  R(z) = R_0 + \sum_{j=1}^{n}(z^{j}R_j + z^{-j}R_j^*).
\end{align} 
and a family $\mathcal{Q} (m, n)$ for it,
\begin{align}
\begin{split}
  \mathcal{Q}(m,n)= \{Q(z) &= \sum_{j=-n}^{n}z^{j}Q_j: \\
     & Q_j =Q_{-j}^*\in \mathbb{R}^{m\times m} ,~\forall \theta \in [-\pi, \pi] \}.
\end{split}
\end{align}
Given $\mathcal{R}$, the block Toeplitz matrix corresponding to $R$ is
\begin{align}
  T(R) = \begin{bmatrix}
           R_0 & R_1 & \cdots & R_n \\
           R_1' & R_0 & \ddots & \vdots \\
           \vdots & \ddots & \ddots & R_1 \\
           R_n' & \cdots & R_1' & R_0
         \end{bmatrix}.
\end{align}
If a matrix $X\in \Qb_{m(n+1)}$ is partitioned as
\begin{align}\label{eq:X}
  X=\begin{bmatrix}
    X_{00} & X_{01} & \cdots & X_{0n} \\
    X_{01}' & X_{11} & \cdots & X_{1n} \\
    \vdots  & \vdots &   & \vdots \\
    X_{0n}' & X_{1n}' & \cdots & X_{nn}
  \end{bmatrix},
\end{align}
then define $D(X):= [D_0(X), \cdots, D_n(X)]$, with
\begin{align}
    D_0(X)=\sum_{h=0}^{n}X_{hh},~D_j(X)=2\sum_{h=0}^{n-j}X_{h,h+j},~j=1,\cdots,n.
\end{align}

{\bf For function spaces.} Denote by $\Lb_2(\Tbb)$ the linear space of $\Cbb^{m\times m}$ matrix functions on the unit circle, where $\Tbb$ is the unit circle. For any function $\Xi(e^{i\theta}) \in \Lb_2(\Tbb)$, we denote $\Xi(e^{i\theta})$ positive definite (or semi-definite) for $\theta \in [-\pi, \pi]$ by $\Xi(e^{i\theta}) \succ 0$ (or $\Xi(e^{i\theta}) \succeq 0$).
Denote by $\mathcal{S}_+^l$ the class of $l\times l$ positive definite spectral densities that are integrable on $[-\pi, \pi]$. The spectral density $\Phi(e^{i\theta})$ of a stationary stochastic process $\zeta(t)$ given its covariance lags is
\begin{align}
    \Phi(e^{i\theta})=\sum_{k=-\infty}^{\infty} R_k e^{ik\theta},
\end{align}
and the covariance lags can be given from 
\begin{align}
    R_k = \int_{-\pi}^{\pi} \Phi(e^{i\theta})e^{ik\theta}  \frac{\text{d} \theta}{2\pi}.
\end{align}

Define the norm of a function $A(e^{i\theta})$ in $\Lb_2(\Tbb)$, as
\begin{align}\nonumber
  \Vert A(e^{i\theta})\Vert = \sup_{\theta \in [-\pi, \pi]}~\sigma_1(A(e^{i\theta})),
\end{align}
where $\sigma_1(A(e^{i\theta}))$ denotes the largest singular value of $A(e^{i\theta})$ at $\theta$.
Define the inner product for two functions in $\Lb_2(\Tbb)$, as
\begin{align}
  \left<A(e^{i\theta}), B(e^{i\theta}) \right> = \int_{-\pi}^{\pi} \trace\left( A(e^{i\theta})B(e^{i\theta})^*\right)\frac{\text{d} \theta}{2\pi}.
\end{align}
Hence the inner product of two constant real matrices is $\left<A, B\right>=\trace (AB')$.
Define the shift operator as 
\begin{align}
  \Delta(e^{i\theta}):= \begin{bmatrix}
                        I & e^{i\theta}I & \cdots & e^{in\theta}I
                      \end{bmatrix}.
\end{align}
Given $X\in \Qb_{m(n+1)}$ in \eqref{eq:X}, by direct computation we have
\begin{align}\label{eq:DeltaXDelta*}
\begin{split}
  \Delta(e^{i\theta}) X \Delta(e^{i\theta})^* &= D_0(X) + \\ 
  & \frac{1}{2}\sum_{j=1}^{n} \left( e^{-ij\theta}D_j(X)+ e^{ij\theta}D_j(X)' \right),
\end{split}
\end{align}
hence $\Delta(e^{i\theta}) X \Delta(e^{i\theta})^* \in \mathcal{Q}(m,n)$.
On the other hand, any element in $\mathcal{Q}(m,n)$ can be expressed by \eqref{eq:DeltaXDelta*}, because $D$ is surjective, leading to
\begin{align}
    \mathcal{Q}(m,n) = \{\Delta(e^{i\theta}) X \Delta(e^{i\theta})^*,~X\in \Qb_{m(n+1)} \}.
\end{align}
In the sequel, we shall often suppress the argument $z$, $z^{-1}$ or $e^{i\theta}$ whenever there is no risk of misunderstanding.

\section{Problem description}\label{sec:problem}
Considering an $(m+l)$-dimensional ($m,l>0$) Gaussian low rank vector identification problem in discrete time, with the measurement model
\begin{align}\label{eq:meas_model}
  \zeta(t) & = y(t) + e(t),
\end{align}
where $\{\zeta(t), {\rm for}~t = 1, \cdots, N\}$ is the observed sequence, $y(t)$ is the Gaussian low rank vector process we want to estimate with zero-mean and an $l$-rank spectral density matrix $\Phi_y(z)$, and $e(t)$ is an $(m+l)$-dimensional zero-mean white Gaussian noise process of covariance matrix $\sigma^2 I$, independent with $y(t)$, with all its entries independent. 
Without loss of generality, we suppose $m> l$ in this paper for problems of a higher dimension. 
Note that in some literature on factor analysis, $e(t)$ is supposed to be an idiosyncratic sequence (see such as \cite{P23}). 
However, as what we shall discuss later, the idiosyncratic assumption can be relaxed to a more basic and general requirement to obtain a consistent estimate.


The problem in this paper is identifying the low rank process $y(t)$, given finite data of the noisy sampling process $\zeta(t)$, i.e., identifying the innovation model for a low rank process based on noisy samplings. 
To the best of our knowledge, this very common phenomenon in practice has not been discussed for the identification of low rank processes.
The issue is 
because of the hidden of the rank deficient spectrum, and possibly emergence of colored noise in estimation resulting from the low rank character.

\subsection{Low rank graphical model}\label{subsec:problem_LRGmodel}
Given the decomposition of $y(t)$ in \eqref{eq:y=ymyl}, 
with $y_l(t)$ of dimension $l$ and a full-rank spectral density matrix $\Phi_{y_l}(z)$, 
the spectrum of $y(t)$ can be partitioned similarly as
\begin{align}\nonumber
  \Phi_y(z) = \begin{bmatrix}
                \Phi_{y_m}(z) & \Phi_{y_{lm}}(z)^* \\
                \Phi_{y_{lm}}(z) & \Phi_{y_l}(z)
              \end{bmatrix},
\end{align}
with $\rank(\Phi_{y_l})=l$.

From our previous work such as \cite{CLPtac23}\cite{CPLauto23}, a deterministic relation always exists between $\Phi_{y_{lm}}(z)$ and $\Phi_{y_{l}}(z)$. 
And there is a {\bf special feedback model} \cite{CLPtac23}, describing the internal relationship between $y_m(t)$ and $y_l(t)$,
\begin{subequations}\label{eq:specialFB}
\begin{align}\label{eq:ym=Hyl}
  y_m(t) &= H(z)y_l(t),\\  \label{eq:yl=Fym}
  y_l(t) &= F(z)y_m(t) + \varepsilon_l(t),
\end{align}
\end{subequations}
where the deterministic relation 
$$
H(z)= \Phi_{y_{lm}}(z)\Phi_{y_{l}}(z)^{-1}, 
$$ 
is casual uniquely determined from $\Phi_y$, 
and $F(z)$ is a (strictly) causal function which can be determined from a minimal realization of $y(t)$ \cite{CLPtac23} or by a one-step Wiener predictor \cite{CPLauto23}.


An illustrative diagram of a special feedback model is given in Fig.~\ref{fig:specialFB}. For processes with invertible spectral densities (i.e., full-rank processes), noise processes $\varepsilon_l$ and $\varepsilon_m$ are both non-zero in general for this feedback structure. 
However, because of the inner correlation, feedback models for low rank processes admit a deterministic forward loop, making the possibly high-dimensional noise $\varepsilon_m\equiv 0$.
Therefore, $y=[y_m',~y_l']'$ can be determined entirely from $y_l$.

\begin{figure}[ht]
\centering
  \includegraphics[scale=0.42]{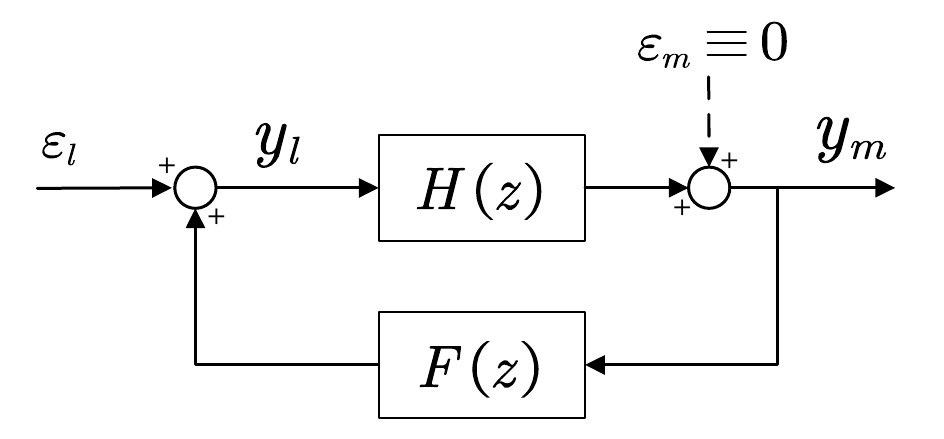}
  \caption{The structure block diagram of a special feedback model \eqref{eq:specialFB}.} \label{fig:specialFB}
\end{figure}

Drag \eqref{eq:ym=Hyl} into the measurement model \eqref{eq:meas_model}, we have
\begin{align}\label{eq:zeta=yl+e}
  \zeta(t) = \begin{bmatrix}
               H(z) \\
               I
             \end{bmatrix}y_l(t) + e(t).
\end{align}
We name \eqref{eq:zeta=yl+e} a {\bf low rank graphical model}, for it obviously corresponds to a latent-variable graphical model \cite{Dahlhaus00, Zorzi16}, where $y_l(t)$ is the lower-dimensional  latent variable. And we name $y_l(t)$ the {\bf low rank latent variable} of this graphical model.
\eqref{eq:zeta=yl+e} can also be seen as a dynamic factor analysis model \cite{EJC10,FAbook10,P23} common in microeconomics, with $y_l$ its minimal factor process.

The model illustration diagrams in identifying \eqref{eq:meas_model} based on low rank graphical models are shown in Fig.~\ref{fig:LowRankGraphical}. 
For a possibly high-dimensional inter-correlated process $y(t)$ measured under noise, the structure block diagram of a general measurement model \eqref{eq:meas_model} is shown in Fig.~\ref{fig:LowRankGraphical} (a). 
By utilizing the low rank characteristic, the measurement model is dimensionally reduced and efficiently simplified into a low rank graphical model \eqref{eq:zeta=yl+e} as shown in Fig.~\ref{fig:LowRankGraphical} (b). $\zeta(t)$ only depends on an $l$-dimensional process, rather than depends on the $(m+l)$-dimensional low rank process $y(t)$. 
Low rank graphical model simplifies the presentation of $\zeta(t)$ and hence will make the estimation procedures easier to be realized. 

{\begin{figure*}[t]
\centering        
    \subfigure[Measurement model]{
        \includegraphics[scale=0.5]{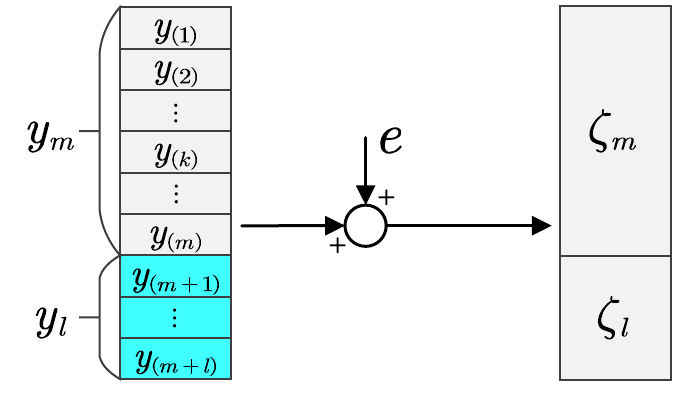}} \hspace{1.6cm}
    \subfigure[Low rank graphical model]{
        \includegraphics[scale=0.5]{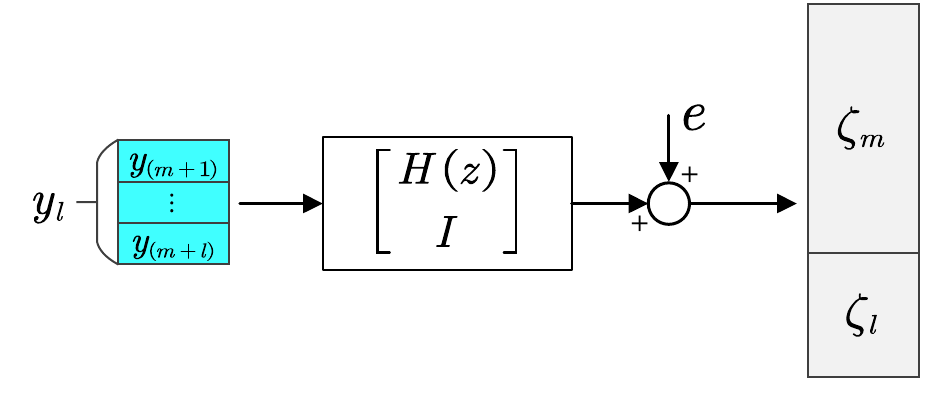}}
    \caption{Structure block diagrams of measurement model \eqref{eq:meas_model} and low rank graphical model \eqref{eq:zeta=yl+e}. By utilizing the low rank characteristic, measurement process $\zeta(t)$ only depends on an $l$-dimensional low rank latent variable $y_l(t)$, rather than the higher-dimensional process $y(t)$. $y_l(t)$ can be seen as the key factor of the hidden low rank process $y(t)$.}
    \label{fig:LowRankGraphical}
\end{figure*}}

Note that the identification problem on \eqref{eq:meas_model} seems similar to the problem of identifying a low rank model with an external input in \cite[Section 6]{CPLauto23}, which is also thoroughly discussed in a series of articles \cite{WVD18}\cite{WVD18-2}, by regarding $y(t)$ as an external input.
But the two problems are quite different since here we have no direct access to the true data of $y(t)$.

We claim that it is worthy studying the estimation of low rank graphical models. 
Different from conventional latent-variable graphical models, $y(t)$ and hence the low rank latent variable $y_l(t)$ in \eqref{eq:zeta=yl+e} are obviously of physical meanings, as the 'real' processes hidden behind the noisy measurements. 
Here the entries of $y_l(t)$ can be seen as the key factors of the low rank process $y(t)$ in problem analysis. 
For amounts of scientific and technology problems, exploring a lower-dimensional  key process with real physical meaning is a crucial mission. 
In contrast, traditional latent-variable graphical models can only extract latent variables of no physical meanings in general, leading that the graphical structure carry less information than low rank graphical models.
It is absolutely more easier to analyze the low rank variable $y_l(t)$ by expert knowledge and measured data from real world,  than an extracted latent variable of no specific meaning. And hence it would be easier to analyze $y_m(t)$ and $y(t)$ based on the extra measured data of $y_l(t)$. 

Special feedback models provide an efficient representation of the internal correlations in low rank processes, by revealing the algebraic (i.e., parametric) and geometric (i.e., topological) relations meantime. 
Based on special feedback models, \eqref{eq:zeta=yl+e} provides a prediction model with respected to (w.r.t.) a minimum-dimensional  latent variable. Deterministic relation $H(z)$ is uniquely determined once $y_l$ is determined \cite{CPLauto23,CLPtac23}. 
In addition, the topology of this low rank graphical model can be determined directly based on the topology in special feedback models as we shall discuss in the following.
Therefore, no further significant effort is required for topology estimation in low rank graphical estimation.

{\subsection{Sparse plus low rank structure}\label{subsec:problem_structure}}
The low rank graphical model \eqref{eq:zeta=yl+e} admits a simplified S+L structure for identification. 
Since $e(t)$ is i.i.d. and independent of $y(t)$, the spectrum of $\zeta(t)$ and its (block) diagonal entries are evidently invertible. 
Denote a block partition of the spectral density of $\zeta(t)$ and its inverse matrix by
\begin{align}\label{eq:Phi,invPhi}
  \Phi=\begin{bmatrix}
         \Phi_{m} & \Phi_{lm}^* \\
         \Phi_{lm} & \Phi_{l}
       \end{bmatrix}, \quad \Phi^{-1}=\begin{bmatrix}
                                        \Upsilon_{m} & \Upsilon_{lm}^* \\
                                        \Upsilon_{lm} & \Upsilon_{l}
                                      \end{bmatrix},
\end{align}
where $\Phi_{m}\in \mathcal{S}_+^{m}$. From Schur complement we have  
\begin{align}\label{eq:invPhi_m}
  \Phi_{m}^{-1} =  \Upsilon_{m} - \Upsilon_{lm}^* \Upsilon_l^{-1} \Upsilon_{lm}.
\end{align}
Then we have the following result.

\begin{thm}\label{lem:S+L}
Equation \eqref{eq:invPhi_m} for the low rank graphical model \eqref{eq:zeta=yl+e}, with $y_l(t)$ the low rank latent variable of dimension $l$, admits a special S+L structure,
i.e., a diagonal plus low rank (D+L) structure,
\begin{align}
   \Phi_{m}^{-1}= \Sigma - \Lambda,\quad \Lambda \succeq 0,
\end{align}
where $\Sigma = \Upsilon_m$ is a diagonal function matrix and $\Lambda = \Upsilon_{lm}^* \Upsilon_l^{-1} \Upsilon_{lm}$ is low-rank. 
\end{thm}

\begin{proof}
To prove the $m\times m$ matrix $\Lambda$ is low-rank, just remind our assumption $m>l$, and the rank of $\Lambda$ is no greater than its factor $\Upsilon_l^{-1}$, i.e., $\rank(\Lambda) \leq \rank(\Upsilon_l^{-1}) \leq l$.

Then we prove $\Sigma$ is diagonal.
Denote the topology of a graph by $(V, E)$, where $E \subseteq V \times V$ is the set of edges,  $V:=\{1,2, \cdots, m+l\}$ is the set of vertices, and denote by $V_l :=\{m+1, \cdots, m+l\} $.
Denote by $\Xb_S = \Span\{\zeta_{(i)}(t):~i\in S,  t\in \Zbb \}$ the vector space generated by $\zeta_{(i)}(t)$ for $i \in S$, node set $S \subseteq V$.  
For any $k, h \in V$ with $k\neq h$, the shorthand notation
\begin{align}\label{eq:condorthognal}
  \Xb_k \perp \Xb_h~\vert~ \Xb_{V\backslash\{k,h\}}
\end{align}
denotes that $\Xb_k$ and $\Xb_h$ are conditionally independent given $\Xb_{V\backslash\{k,h\}}$.

From \cite{Lin13}, we have
\begin{align}\label{eq:condIndep}
\begin{split}
  (k, h) \notin E  & \Leftrightarrow k \neq h,  \Xb_k \perp \Xb_h~\vert~ \Xb_{V\backslash\{k,h\}} \\
   & \Leftrightarrow [\Phi(e^{i\theta})^{-1}]_{kh} = 0, \text{for~} \theta \in [-\pi, \pi].
\end{split}
\end{align}
Partition $H(z)$ as
\begin{align}\label{eq:Hpartition}
  H(z) =\begin{bmatrix}
          H_1(z) \\
          \vdots \\
          H_m(z)
        \end{bmatrix},
\end{align}
from \eqref{eq:zeta=yl+e} we have 
\begin{align}\label{eq:zeta_(i)}
  \zeta_{(j)}(t) = H_j(z)y_l(t) + e_{(j)}(t),\quad \text{for~}j\leq m.
\end{align}
Then for $k \neq h, k, h\leq m$,
\begin{align}\nonumber 
\begin{split}
  \Xb_j  - \mathbb{E}^{\Xb_{V\backslash \{k,h\}}}&\Xb_j =  \Xb_j - \mathbb{E}^{\Xb_{V_l}} \Xb_j \\
   &= \Span\{e_{(j)}(t), t\in \Zbb\},\quad \text{for~}j=k,h.
\end{split}
\end{align}
And obviously $\Span\{e_{(k)}(t), t\in \Zbb\} \perp \Span\{e_{(h)}(t), t\in \Zbb\}$.
Hence we have
\begin{align}
\begin{split}
   <\alpha - \mathbb{E}^{\Xb_{V\backslash \{k,h\}}}\alpha,~ &
  \beta  - \mathbb{E}^{\Xb_{V\backslash \{k,h\}}}\beta> = 0, \\
  &~ \forall \alpha \in \Xb_k, \forall \beta \in \Xb_h,
\end{split}
\end{align}
meaning \eqref{eq:condorthognal} holds \cite[pp. 38-40]{LPbook}, and hence $\Sigma$ is diagonal.
\end{proof}

Besides Theorem~\ref{lem:S+L}, it is straightforward to prove that $\Upsilon_l$ is diagonal as well.  
Note that, in a traditional S+L structure, matrix $\Sigma$ is sparse with non-zero entries instead of diagonal. Our feedback modeling of low rank process $y(t)$ simplifies the structure and hence the further calculations.
Then the space of $\Phi_m^{-1}$ can be decomposed into the following two subspaces,
\begin{subequations}\label{eq:VEmVG}
\begin{align}
  \Vb_{Em} &= \{\Delta X\Delta^* \in \mathcal{Q}(m,n), X\in \Db_{m(n+1)} \}, \\
  \Vb_{G}  &= \{ \Delta GG'\Delta^*,~\text{s.t.}~G \in \Rbb^{m(n+1)\times l}  \},
\end{align}
\end{subequations}
where $X\in \Db_{m(n+1)}$ is naturally diagonal by Theorem~\ref{lem:S+L}.

Therefore, the topology of a low rank graphical model is simplified to a large extent compared with a general graphical or network model.
One example for the topology of a low rank graphical model is shown in Fig.~\ref{fig:exmTopology}. 
Based on the special feedback model \eqref{eq:specialFB}, the topology of a general low rank graphical model \eqref{eq:zeta=yl+e} can be expressed by a latent-variable graph (or a directed tree) $(V^{\zeta}, V^{y_l}, E)$, with measured (or child) node set
\begin{subequations}
\begin{align}
  \zeta_{(h)} \in V^{\zeta}= \{1,2, \cdots ,m+l \}, ~~\text{for~} h=1, \cdots,m+l,
\end{align}
low rank latent-variable (or root) node set
\begin{align}\label{eq:Vyl}
    y_{l(k)} \in V^{y_l}= \{1,2, \cdots ,l \}, ~~\text{for~} k=1, \cdots, l,
\end{align}
and $E \subseteq V^{\zeta} \times V^{y_l}$ the directed edge set from $V^{y_l}$ to $V^{\zeta}$, satisfying the sparsity condition
\begin{align}\label{eq:edgesLRG}
    (h,k) \notin E, \text{~~for~} h\neq j+m \text{~and~} ~h>m, 
\end{align}
\end{subequations}
where the special D+L structure is clearly demonstrated by the sole existing edge of $\zeta_{(h)}$ from $y_{(h)}$ for $h>m$, as also depicted in Fig.~\ref{fig:exmTopology}. 

\begin{figure}[ht]
  \centering
  \includegraphics[scale=0.46]{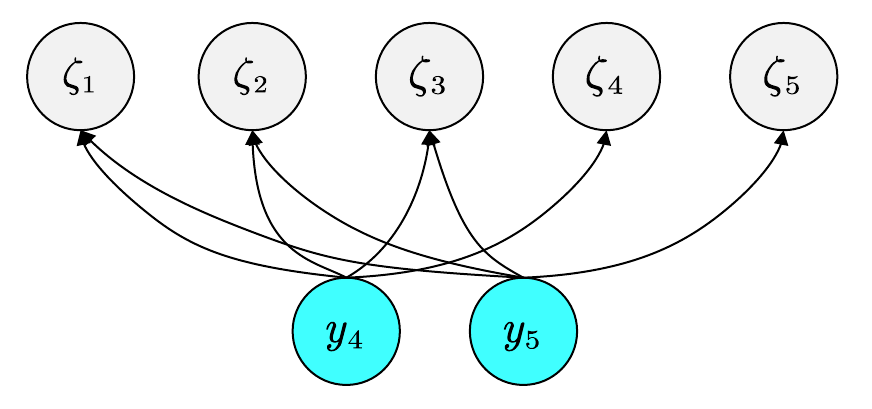}
  \caption{The topology of a low rank graphical model \eqref{eq:zeta=yl+e}, with $\zeta(t)$, $y(t)$, $y_l(t)$ of dimension $5$, $5$, $2$ respectively. }\label{fig:exmTopology}
\end{figure}

From Theorem~\ref{lem:S+L}, the identification of low rank vector processes under noise is actually a S+L estimation problem, or more precisely, a D+L estimation problem on the low rank graphical model \eqref{eq:zeta=yl+e}. 
Traditionally, the topology of a latent-variable graphical model requires to be estimated during the S+L estimation procedure. 
However, in our case, the general topology of the low rank graphical model is directly obtained from a special feedback model, eliminating the need for topology estimation.

Several studies have been conducted on the estimation of latent-variable graphical models effectively, see such as \cite{Zorzi16,Yu23cdc,Yu24auto}. 
Existing research initially focused on estimating AR \cite{Zorzi16,Yu24auto} and later expanded to ARMA latent-variable structures \cite{Yu23cdc}, often assuming simplified polynomial matrices (like constant matrices rather than general polynomials), to facilitate deduction and computation in algorithm design. 
However, we shall show in the subsequent discussion that low rank graphical models indeed admit a specific ARMAX graphical structure, rendering the existing techniques inapplicable due to the lack of accounting for both the moving average and exogenous input components in estimation.
In the paper, we address the identification issue of low rank graphical models, including proposing an algorithm for solving the ARMAX graphical estimation problem. 
This also offers another angle for effectively estimating general ARMAX latent-variable graphical models. 
The corresponding identification algorithm will be proposed in the following sections.

\subsection{Low rank ARMAX graphical model}\label{subsec:problem_ARMAX}
Denote $H(z)$ by
\begin{equation}\label{eq:H=A-1B}
H(z)=A(z^{-1})^{-1}B(z^{-1}),
\end{equation}
where $A(z^{-1})$ and $B(z^{-1})$ are $m\times m$, $m \times l$ left-coprime polynomial matrices of $z^{-1}$. Without loss of generality, suppose $A(z^{-1})$ is invertible and monic. 
Partition $\zeta(t)$ and $e(t)$, 
\begin{align}
  \zeta(t)= \begin{bmatrix}
             \zeta_m(t) \\
              \zeta_l(t)
            \end{bmatrix}, \quad e(t)=\begin{bmatrix}
             e_m(t) \\
              e_l(t)
            \end{bmatrix},
\end{align}
similarly as in \eqref{eq:y=ymyl}.
Since the data $\zeta(t)$ can be seen as the measurements of $y(t)$ under noise, a least square identification procedure \cite{Ljung} can be directly used to estimate $H(z)$, from
\begin{align}\label{eq:H(z)ARMAX}
  A(z^{-1})\zeta_m(t) = B(z^{-1})\zeta_l(t) + v(t),
\end{align}
where $v(t)$ is the estimation error process of the identification procedure.
This approach has been proven to have a good performance converging to the true value of $H(z)$ given the degrees, in the absence of noise $v(t)$ in our previous work \cite{CPLauto23}.

However, as we shall demonstrate in the following analysis and through Example~1 in Section~\ref{sec:example},
identifying by least squares yields only rough estimates. 
The primary reason for this limitation is that $v(t)$ actually is not a white noise.
From \eqref{eq:zeta=yl+e} we have
\begin{align}\nonumber
\begin{split}
  \begin{bmatrix}
    A(z^{-1}) &  \\
     & I
  \end{bmatrix}\begin{bmatrix}
                 \zeta_m(t) \\
                 \zeta_l(t) 
               \end{bmatrix} = & \begin{bmatrix}
                           B(z^{-1}) \\
                           I
                         \end{bmatrix}y_l(t) \\ 
                          & +  \begin{bmatrix} A(z^{-1}) &  \\ & I\end{bmatrix}\begin{bmatrix}
                 e_m(t) \\
                 e_l(t) 
               \end{bmatrix},
\end{split}
\end{align}
and hence
\begin{subequations}\label{eq:H(z)2lines}
\begin{align}\label{eq:ARMAX}
A(z^{-1})\zeta_m(t) &= B(z^{-1})y_l(t) + A(z^{-1})e_m(t), \\
\label{eq:zetal=yl+el}
 \zeta_l(t) &=y_l(t)+  e_l(t).
\end{align}
\end{subequations}
Then we obtain an ARMAX-like relation between the subprocesses of the observed process $\zeta(t)$, i.e.,
\begin{align}\nonumber 
  A(z^{-1})\zeta_m(t)=B(z^{-1})\zeta_l(t) + A(z^{-1})e_m(t)-B(z^{-1})e_l(t),
\end{align}
where we find that the error $v(t)$ in \eqref{eq:H(z)ARMAX} equals to
\begin{align}\nonumber
  v(t)= A(z^{-1})e_m(t)-B(z^{-1})e_l(t),
\end{align}
which is determined by two white-noise vector subprocesses and the parameter polynomial matrices. 
Least squares hence can not lead to a consistent estimate based on such a non-white noise.

Nonetheless, based on \eqref{eq:H(z)2lines} we are able to decompose the low rank graphical identification challenge into two subproblems in sequel:
\begin{enumerate}
  \item Identifying an innovation model for the low rank latent variable $y_l(t)$ based on noisy measurement $\zeta_l(t)$ in \eqref{eq:zetal=yl+el}.
  \item Identifying $A(z^{-1}), B(z^{-1})$ based on the estimates of $y_l(t)$ in the specific ARMAX model \eqref{eq:ARMAX} .
\end{enumerate}

\section{Identification of low rank latent variable $y_l$}\label{sec:yl}

From \eqref{eq:zetal=yl+el}, since $y(t)$ is independent of the i.i.d. process $e(t)$, the covariance lags of $\zeta_l(t)$ are
\begin{align}\nonumber
\begin{split}
    R_j^{l} &= \E \left[ \left(y_l(t+j)+  e_l(t+j)\right)\left(y_l(t)+  e_l(t)\right)' \right] \\
    & = \E \left[ y_l(t+j)y_l(t)' \right] + \E \left[ e_l(t+j)e_l(t)' \right],
    \end{split}
\end{align}
for $j=1,2, ..., n$. 
Hence, the $n$-dimensional covariance lag vector of $\zeta_l$ satisfies
\begin{align}\label{eq:Rcal}
    \mathcal{R}^l = \mathcal{R}^{y_l} + \begin{bmatrix}
                      \sigma^2 I & 0 & \cdots & 0 
                    \end{bmatrix}.
\end{align}
Although $y_l(t)$ is measured under noise and not accessible, by \eqref{eq:Rcal} its covariance lags can be calculated from those of the noisy measurement process. 
Therefore, an estimation approach based on the covariance extension of $y_l$ can be proposed for estimating low rank latent variables.

\subsection{Maximum entropy covariance extension}
Given a series of measurement data $\{\zeta(1), \cdots, \zeta(N) \}$, specifically, we shall use maximum entropy method for covariance extension to estimate the spectrum and an innovation model for $y_l$. 
Then, the maximum entropy covariance extension problem of $y_l$ 
can be expressed by a strictly convex program 
\begin{subequations}\label{eq:maxiEntropy}
\begin{align}
  &\hat{\Phi}_{y_l}^\circ (e^{i\theta})= \mathop{\arg \max}\limits_{\Phi_{y_l} \in \mathcal{S}_+^l} \int_{-\pi}^{\pi} \log \det \Phi_{y_l}(e^{i\theta}) \frac{\text{d} \theta}{2\pi}, \\
  & \text{subject to}\quad\int_{-\pi}^{\pi}  \Delta(e^{i\theta})\Phi_{y_l}(e^{i\theta}) \frac{\text{d} \theta}{2\pi} = \hat{\mathcal{R}}^{y_l},
\end{align}
where
\begin{align}\label{eq:R_yl}
    \hat{\mathcal{R}}^{y_l}:=\hat{\mathcal{R}}^{l} - \begin{bmatrix}
                      \sigma^2 I & 0 & \cdots & 0 
                    \end{bmatrix},
\end{align}
\end{subequations} 
denotes the estimated covariance lags of $y_l$, $\hat{\mathcal{R}}^{l}$ denotes the estimated covariance lags of process $\zeta_l(t)$, and  $\hat{\Phi}_{y_l}^\circ (e^{i\theta})$ is named the maximum-entropy covariance extension. Note that $\hat{\mathcal{R}}^{y_l}$ of $y_l(t)$ is defined differently from the general estimated covariance lags in \eqref{eq:hatRj} for lack of the real data.
Denote by $\hat{R}_j$ the $j$-th estimated covariance lag of $y_l(t)$ in \eqref{eq:R_yl}.

Then referring to \cite{GraphicalBook10} or \cite{Lin13}, the dual of problem \eqref{eq:maxiEntropy} is 
\begin{subequations}\label{eq:dualMaxiEntropy}
\begin{align}
   \min\limits_{\Phi_{y_l}^{-1}\in \mathcal{Q}(l,n)}~ &-\int_{-\pi}^{\pi} \log \det \Phi_{y_l}^{-1} \frac{\text{d} \theta}{2\pi} + \left< \Phi_{y_l}^{-1}, \hat{\Phi}_{y_l}\right> ,\\
   \text{subject to}\quad &\Phi_{y_l} \succ 0,
\end{align} 
\end{subequations} 
where $\Phi_{y_l}^{-1}$ is exactly the Lagrangian multiplier, $\hat{\Phi}_{y_l} := \sum_{j=-n}^{n}e^{-ij\theta}\hat{R}_j$ is the $n$-length windowed correlogram of $y_l$. 

Then a specific estimation approach can be used to estimate $y_l(t)$. There have been quite a few references on estimating through maximum entropy covariance extensions methods, such as \cite{Georgiou02,Lin99,Picci11}.

\subsection{AR graphical estimation}
In this paper, to obtain the estimated time series of $y_l(t)$ during the sampling period, we identify an AR innovation model for $y_l(t)$,
\begin{align}\label{eq:ARg}
    P_0 y_l(t) = - \sum_{j=1}^{n} P_j y_l(t-j) + w(t),
\end{align}
with $P_0$ 
invertible, $w(t)$ an i.i.d. $l$-dimensional zero-mean Gaussian white noise with variance matrix $I$. 
Note that, we choose the above model instead of restricting $P_0=I$, to avoid estimating the variance matrix of the innovation process $w(t)$.

To simplify the calculations in solving the dual problem \eqref{eq:dualMaxiEntropy}, we introduce the following topology restriction to make the multiplier $\Phi_{y_l}^{-1}$ as sparse as possible. 
Suppose we are given an internal interaction graph of $y_l(t)$ itself, $(V^{y_l}, E^{y_l})$, with $y_{l(k)}\in V^{y_l}$ given in \eqref{eq:Vyl}. 
From \eqref{eq:condIndep}, if $y_{l(k)}$ and $y_{l(h)}$ are conditionally independent in $y_l$, the equivalent restriction on $\Phi_{y_l}^{-1}$ is
\begin{align}\label{eq:orginRsctAR}
    \left[\Phi_{y_l}^{-1} \right]_{kh} = 0,\quad \text{for } (k,h) \notin E^{y_l}.
\end{align}
Note that the self loop from $y_{l(k)}$ to itself is seen as an existing edge in its internal topology, i.e., $(k,k) \in E^{y_l}$ for $k=1, \cdots, l$.

Define
\begin{align}\label{eq:F(z)}
    &P(z^{-1}) := \sum_{j=0}^{n} P_j z^{-j},\\
    \label{eq:mathcalF}
    & \mathcal{P}:=\begin{bmatrix}
                P_0 & P_1 & \cdots & P_n
              \end{bmatrix}.
\end{align}
Since $y_l$, $w$ are full-rank processes, $P(z^{-1})$ is invertible. From $y_l(t)= P(z^{-1})^{-1}w(t)$, we easily have 
\begin{align}\label{eq:Phi_yl-1}
    \Phi_{y_l}^{-1} = P(z^{-1})^{*}P(z^{-1}).
\end{align}

Then from direct calculations using \eqref{eq:F(z)}\eqref{eq:mathcalF}\eqref{eq:Phi_yl-1}, and based on Kolmogorov's formula, dual problem \eqref{eq:dualMaxiEntropy} with topology restriction \eqref{eq:orginRsctAR} reduces to
\begin{subequations}\label{eq:dualARgraph}
\begin{align}\label{eq:prob_dualARgraph}
    \min\limits_{\mathcal{P}}~ &-\log \det P_0'P_0  + \left< T(\hat{\mathcal{R}}^{y_l}), \mathcal{P}'\mathcal{P} \right> ,
\end{align} 
subject to
\begin{align}\label{eq:cons_dualARgraph}
\begin{split}
   \left[D_j (\mathcal{P}'\mathcal{P}) \right]_{kh}=0, &~~ \text{for } (k,h) \notin E^{y_l},  \\
       &\text{and for } j=0,1, \cdots, n,
\end{split}
\end{align}
\end{subequations}
where $\hat{\mathcal{R}}^{y_l}$ is given by \eqref{eq:R_yl}. 

Due to the symmetry of spectral density matrices, the topology constraints \eqref{eq:cons_dualARgraph} can be further reduced to
\begin{align}\label{eq:cons2_dualARgraph}
\begin{split}
 \left[D_j (\mathcal{P}'\mathcal{P}) \right]_{kh}=0, ~~ &\text{for } (k,h) \notin E^{y_l} \text{ and } k\geq h,  \\
      & \text{and for } j=0,1, \cdots, n.
\end{split}
\end{align}
 
The dual problem~\eqref{eq:dualARgraph} and the convex problem \eqref{eq:maxiEntropy} under constraint \eqref{eq:orginRsctAR} have unique solutions \cite{Lin13}.  
For a further simplification, solution and the properties of the convex problem \eqref{eq:dualARgraph}, we recommend \cite{Dahlhaus00}, \cite[pp. 98-109]{GraphicalBook10}, \cite{ConvexBook04} for readers as references.
This paper will not dwell on the topology estimation of AR graphical models, plentiful references can be given such as  \cite{Zorzi16,Lin13,topo10}.

\subsection{Estimation of the low rank latent variable}
Traditionally, given the estimates $\hat{P}_j$ for $j=0,1,\cdots,n$, $y_l(t)$ can be iteratively estimated (like in \cite{GraphicalBook10}) by
\begin{align}\label{eq:ylhat_old}
    \hat{y}_l(t) = - \hat{P}_0^{-1}\sum_{j=1}^{n} \hat{P}_j \hat{y}_l(t-j).
\end{align} 
However, if the covariance norm of the low rank latent variable $y_l$ is close to 1, that is, close to the covariance norm of innovation process $w(t)$, a non-negligible estimation error in each step will lead to a substantial accumulated error over time, resulting in an inaccurate estimation or prediction.

Therefore, though graphical estimation based on maximum entropy covariance extension usually gives good parameter estimates, a refined estimating approach of $y_l(t)$ instead of \eqref{eq:ylhat_old} is necessary for the sequential ARMAX estimation.
In this paper, we give a simple but efficient way of adjusting the estimation of $y_l(t)$ iteratively by
\begin{subequations}\label{eq:ylhat_new}
\begin{align}\label{eq:ylhat_new=}
  \hat{y}_l(t) = -\hat{P}_0^{-1}\sum_{j=1}^{n} \hat{P}_j \hat{y}_l(t-j) + \hat{w}_{c}(t) , ~~\text{for~} t> n,
\end{align}
with initial estimated values
\begin{align}
 \hat{y}_l(t)= \zeta_l(t), ~~\text{for~} t=1,\cdots, n,
\end{align}
where $\hat{w}_c(t)$ is a compensation variable estimating the process $P_0^{-1}w(t)$ at time $t$,
\begin{align}\label{eq:w_c}
  \hat{w}_c(t)= \frac{1}{1+\sigma^{2l}\det \hat{P}_0'\hat{P}_0}\left(\zeta_l(t) + \hat{P}_0^{-1}\sum_{j=1}^{n} \hat{P}_j \hat{y}_l(t-j)\right).
\end{align}
\end{subequations}
The accuracy of \eqref{eq:ylhat_new} will be verified in Section~\ref{sec:example}, and might be taken a glance from the following analysis.

Given the real parameters and real data, 
\begin{align}\nonumber
  \hat{w}_c(t)=\frac{\det \left( {P}_0^{-1}(P_0')^{-1}\right) }{\sigma^{2l}+\det\left( {P}_0^{-1}(P_0')^{-1}\right)} \left( e(t) -P_0^{-1}w(t) \right),
\end{align}   
can be seen as taking a part from the sum of processes $e(t)$ and $-P_0^{-1}w(t)$ by the ratio of the fractional coefficient in the equation,
where we have $\det \E \{e(t)e(t)' \}=\sigma^{2l}$, $\det \E \{{P}_0^{-1}w(t)w(t)'({P}_0^{-1})' \}= \det \left( {P}_0^{-1}(P_0')^{-1}\right)$. 
In addition, it is easily checked $\E\{\hat{w}_c(t)\}=0$ and $\E\{\hat{y}_l(t)-y_l(t)\}=0$ given the real parameters, with $\hat{y}_l$ given in \eqref{eq:ylhat_new}.

\section{Low rank graphical ARMAX estimation}\label{sec:ARMAX}
Based on \eqref{eq:H=A-1B}, denote the parameter to estimate by 
\begin{subequations}\label{eq:AcalBcalAzBz}
\begin{align}\label{eq:Acal}
    \mathcal{A} :=& \begin{bmatrix}
               A_1 & A_2 & \cdots & A_q
              \end{bmatrix}_{m\times mq},\\\label{eq:Bcal}
    \mathcal{B} := & \begin{bmatrix}
                    B_0 & B_1 &  \cdots & B_r
             \end{bmatrix}_{m\times l(r+1)},
\end{align}
with
\begin{align}\label{eq:AzBz}
    A(z^{-1})= \sum_{k=0}^{q}A_k z^{-k}, \quad B(z^{-1})= \sum_{k=0}^{r}B_k z^{-k},
\end{align}
\end{subequations}
where $A_0=I$ from our assumption.

If we are given the data of $\zeta_m(t)$, ${y}_l(t)$ for $t=1,2,\cdots, N$, we have the model to identify in \eqref{eq:ARMAX}, i.e., 
\begin{align}\nonumber
    A(z^{-1})\zeta_m(t) = B(z^{-1}){y}_l(t) +  A(z^{-1})e_m(t).
\end{align} 
This leads to an ARMAX estimation problem with a specific restriction that process $\zeta_m$ and the noise process share a common parameter matrix. 
To the best of our knowledge, discussion on the identifiability of such a special ARMAX model is still lacking.
Therefore, before solving the problem, we shall firstly explore the identifiability of model \eqref{eq:ARMAX}.

\subsection{Identifiability of the specific ARMAX estimation}\label{subsec:ARMAX_identifiability}
Similar to but different from the identifiability of ARMAX models with one-step delay given by Theorem~\ref{thm:chen93} in Appendix~\ref{apdx:thms} (see also \cite[Theorem 2.1]{Chen93}) and its generalized counterpart Corollary~\ref{coro:chen93} we extend in Appendix~\ref{apdx:thms}, the following result for our ARMAX model \eqref{eq:ARMAX} can be given. 

\begin{thm}[Identifiability of low rank ARMAX graphical model]\label{thm:identifiability}
The system \eqref{eq:ARMAX} is identifiable, if and only if $A, B$ are left coprime and $\rank(\begin{bmatrix} A_q & B_r \end{bmatrix})=m$.
\end{thm} 

\begin{proof}
Recall the definition of identifiability: for a certain pair of polynomials $(A(z^{-1}), B(z^{-1}))$ with respect to $z^{-1}$, there is no polynomials $\tilde{A}(z^{-1})=\sum_{k=0}^{\tilde{q}}\tilde{A}_k z^{-k}$, $\tilde{B}(z^{-1})=\sum_{k=0}^{\tilde{r}}\tilde{B}_k z^{-k}$, with $\tilde{q} \leq q$, $\tilde{r} \leq r$, respectively, such that $\tilde{A}(z^{-1})^{-1}\tilde{B}(z^{-1}) \equiv A(z^{-1})^{-1}B(z^{-1})$ unless $\tilde{A}(z^{-1})\equiv A(z^{-1})$,  $\tilde{B}(z^{-1})\equiv B(z^{-1})$.

$\Rightarrow$: When \eqref{eq:ARMAX} is identifiable, suppose $A, B$ have a same left factor $D(z^{-1})\neq I$ with respect to $z^{-1}$. Since $A$ is invertible, $D$ is invertible.
Then we have $\begin{bmatrix}\tilde{A} & \tilde{B}\end{bmatrix}:= D(z^{-1})\begin{bmatrix}A & B\end{bmatrix}$, where $\tilde{A}$, $\tilde{B}$ are polynomial matrices, satisfying $\tilde{A} \neq A$, $\tilde{B} \neq B$, and the degrees of $\tilde{A}$ and $\tilde{B}$ are not greater than $A$, $B$, respectively. This contradiction proves that the left coprimeness is a necessary condition of identifiability.

Suppose $\rank(\begin{bmatrix} A_q & B_r \end{bmatrix}) < m$. Then there exists a (non-unique) non-zero $m\times m$ constant matrix $D$, satisfying $D\begin{bmatrix} A_q & B_r \end{bmatrix}=0$. Denote by
\begin{subequations}\nonumber
\begin{align}
\begin{split}
    \tilde{A}(z^{-1}):= & (I+Dz^{-1})A(z^{-1})= A_0 + (A_1 + DA_0)z^{-1} \\
   & + \cdots + (A_q + DA_{q-1})z^{-q},
\end{split}
\end{align}
\begin{align}
\begin{split}
    \tilde{B}(z^{-1}):= & (I+Dz^{-1})B(z^{-1})= B_0 + (B_1 + DB_0)z^{-1}\\
    & + \cdots + (B_r + DB_{r-1})z^{-r}.
\end{split}
\end{align}
\end{subequations}
We have 
$$\tilde{A}(z^{-1})^{-1}\tilde{B}(z^{-1}) \equiv A(z^{-1})^{-1}B(z^{-1}).$$
However, from the definition of identifiability, the above equation can only lead to $\tilde{A}\equiv A, \tilde{B} \equiv B$, and $D=0$. This contradiction shows that $\rank(\begin{bmatrix} A_q & B_r \end{bmatrix})=m$.

$\Leftarrow$: If $A, B$ are left coprime and $\rank(\begin{bmatrix} A_q & B_r \end{bmatrix})=m$.
Suppose there exists a pair of polynomials $\begin{bmatrix}\tilde{A} & \tilde{B}\end{bmatrix} \neq \begin{bmatrix}A & B\end{bmatrix}$, with degrees $\tilde{q}\leq q$, $\tilde{r}\leq r$, respectively, satisfying 
$$
\tilde{A}(z^{-1})^{-1}\tilde{B}(z^{-1}) \equiv A(z^{-1})^{-1}B(z^{-1}).
$$ 

Denote by $D(z^{-1}):= \tilde{A}^{-1}A$, from the above equation, we have $\begin{bmatrix}A & B\end{bmatrix} = D(z^{-1})\begin{bmatrix}\tilde{A} & \tilde{B}\end{bmatrix}$. Since $\tilde{A} \neq A$ with a degree no greater than $A$, $D(z^{-1})$ is a non-identity polynomial matrix with respect to $z^{-1}$, leading that $A, B$ have a same left factor $D(z^{-1})$. This contradicts the coprimeness. 

Denote by $\tilde{D}(z^{-1}):= A^{-1}\tilde{A}$, then we have $\begin{bmatrix}\tilde{A} & \tilde{B}\end{bmatrix} = \tilde{D}(z^{-1})\begin{bmatrix}A & B\end{bmatrix}$ and $\tilde{D}$ non-identity. Since $A, B$ are left coprime, there exist two polynomial matrices $M_1(z^{-1}), M_2(z^{-1})$, such that 
$$
    A(z^{-1})M_1(z^{-1}) + B(z^{-1})M_2(z^{-1})=I,
$$
and hence $\tilde{A}(z^{-1})M_1(z^{-1}) + \tilde{B}(z^{-1})M_2(z^{-1})=\tilde{D}(z^{-1})$, implying that $\tilde{D}(z^{-1})$ is a polynomial matrix. Set $\tilde{D}(z^{-1})=\sum_{k=0}^{n_d}D_k z^{-k}$ (with $D_{n_d}\neq 0$), we have 
\begin{subequations}\nonumber
\begin{align}
\begin{split}
   & \begin{bmatrix}
      \tilde{A}(z^{-1}) & \tilde{B}(z^{-1})
    \end{bmatrix} = \begin{bmatrix} D_0A_0 & D_0B_0 \end{bmatrix} \\ 
    &+ \begin{bmatrix} D_0A_1 + D_1A_0 &  D_0B_1 + D_1B_0 \end{bmatrix}z^{-1} 
                      + \cdots \\
                    & + \begin{bmatrix} D_{n_d}A_\alpha &  D_{n_d}B_\alpha \end{bmatrix}z^{-\alpha},
\end{split}
\end{align}
\end{subequations}
with $\alpha=\max\{q,r\}$.  
Note that $A_k=0$ for $k>q$ and $B_j=0$ for $j>r$. 
If $n_d=0$, $A, B$ will have a common factor. Hence $n_d \geq 1$. 
Since $\tilde{q}\leq q$, $\tilde{r}\leq r$, we must have $D_{n_d} A_q=0$, $D_{n_d} B_r=0$, i.e., $D_{n_d} \begin{bmatrix} A_q & B_r \end{bmatrix}=0$ with $D_{n_d}\neq 0$. The two contradictions prove the sufficiency. 
\end{proof}

Our low rank ARMAX model \eqref{eq:ARMAX} does not satisfy the sufficient and necessary conditions of general ARMAX identifiability in Corollary~\ref{coro:chen93}, since the parameter matrices corresponding to $\zeta_m$ and $e_m$ are the same (i.e., not left-coprime in general). 
However, the sufficient and necessary conditions can still be given for the low rank graphical ARMAX identifiability in Theorem~\ref{thm:identifiability}. 
Also note that, the above theorem holds when $\zeta(t)=0, \hat{y}_l(t)=0, e_m(t)=0$ for $t\leq 0$, which is an implicit precondition of this paper.

Since we have assumed the left coprimeness between polynomial matrices $A$ and $B$, system \eqref{eq:ARMAX} is identifiable iff  $\rank(\begin{bmatrix} A_q & B_r \end{bmatrix})=m$.

\subsection{Maximum likelihood estimation}\label{subsec:ARMAX_ML}
Next we shall use maximum likelihood estimation to estimate the low rank graphical ARMAX model \eqref{eq:ARMAX}. Denote by 
\begin{subequations}\label{eq:x(t)dist}
\begin{align}
x(t):= A(z^{-1})e_m(t) = A(z^{-1})\zeta_m(t) - B(z^{-1})y_l(t),
\end{align}
for $t=1,2,\cdots, N$. By assuming $e(t)\sim N(0,\sigma^2I)$, we have  
\begin{align}
&x(t) \sim N(0,\Gamma),\\ \label{eq:Gamma}
&\Gamma:= \sigma^2\sum_{k=0}^{q}A_kA_k'= \sigma^2(I + \Acal \Acal'),
\end{align} 
\end{subequations}
with $\Acal$ given in \eqref{eq:Acal}. Hence we have $\Gamma \succ 0$.

From \eqref{eq:Hpartition} and \eqref{eq:zeta_(i)}, the $j$-th row of $H(z)$ can be written as 
\begin{align}\label{eq:H(i)A(i)B(i)}
  H_j(z)=A_{(j)}^{-1}(z^{-1})B_{(j)}(z^{-1}),
\end{align}
where $A_{(j)}(z^{-1})$ and $B_{(j)}(z^{-1})$ are left coprime, with $A_{(j)}$ a scalar monic polynomial and $B_{(j)}$ a $1\times l$ polynomial vector. 
Hence without loss of generality, we can formulate $A(z^{-1})$ and $B(z^{-1})$ as
\begin{subequations}
\begin{align} \label{eq:Adiag}
A(z^{-1})&=\diag\left\{ A_{(1)}(z^{-1}),~ \cdots,~ A_{(m)}(z^{-1}) \right\},\\
B(z^{-1})&=\begin{bmatrix}
             B_{(1)}(z^{-1})' & \cdots & B_{(m)}(z^{-1})'
           \end{bmatrix}'.
\end{align}
\end{subequations}
Then by direct calculations, \eqref{eq:Adiag} leads to a constraint in parameter estimation
\begin{align}\label{eq:Acons}
    A_j \in \Db_m, \text{~~for~} j=1,\cdots, q.
\end{align}

Denote by 
\begin{align}\label{eq:ThetaAcalBcal}
    \Theta:=[\Acal, \Bcal], 
\end{align}
with $\Acal, \Bcal$ given in \eqref{eq:Acal}\eqref{eq:Bcal}. From \eqref{eq:x(t)dist}, the conditional probability density function of each $x(t)$ is 
\begin{align}\label{eq:likelihood}
    p(x(t) | \Theta) = \frac{1}{(2\pi)^{m/2} \left(\det \Gamma\right)^{1/2}} e^{- \frac{1}{2} x(t)'\Gamma^{-1}x(t)}, 
\end{align}
for $t=1, \cdots, N$. 
The log-likelihood function can be obtained as
\begin{align}
\begin{split}
    \log p(x(1), \cdots, x(N)| \Theta) = &-\frac{mN}{2} \log(2\pi) \\
    - \frac{N}{2}\log(\det \Gamma) 
   & - \frac{1}{2} \sum_{t=1}^{N} x(t)'\Gamma^{-1}x(t),
\end{split}
\end{align}
leading to an optimization problem of our maximum likelihood estimation procedure 
\begin{subequations}\label{eq:MLOriginalProblem}
\begin{align}
  \hat{\Theta}= \arg \min_{{\Theta}}~J_N({\Theta}),
\end{align}
subject to \eqref{eq:Acons},
where 
\begin{align}\label{eq:J(A,B)}
  J_N({\Theta}) = \log(\det \Gamma(\Theta))+ \frac{1}{N}\sum_{t=1}^{N} x(t,\Theta)'\Gamma^{-1}(\Theta)x(t,\Theta),
\end{align} 
\end{subequations} 
$x(t)$, $\Gamma$ are given in \eqref{eq:x(t)dist}.

In the above optimization problem we only consider a constraint \eqref{eq:Acons} for parameter $\Theta$. 
In specific problems or systems, some of the entries in $H(z)$ might be zero, corresponding to the non-existence of some edges in the specific graph of $H(z)$ corresponding to an edge set $E^{H(z)}$.
If $[H(z)]_{kh}$ is given or estimated zero for $(k,h) \notin E^{H(z)}$, the following constraint can be added to the optimization problem, 
\begin{align}\label{eq:Bcons}
\begin{split}
    [B_j]_{kh}=0, ~~&\text{for~} j=0,1,\cdots, r, \\
                & ~~\text{for~} (k,h)\notin E^{H(z)},
\end{split}
\end{align}
without changing the convexity of the feasible domain. 
And hence the corresponding constrained problem is obviously feasible.

The absence of the edges in a low rank ARMAX graphical model (i.e., the zero entries in $H(z)$) could be estimated by designing evaluation indices like in \cite{Lin13,Zorzi16}.
However, because the main aim of the paper is to introduce the low rank graphical model and to propose one integrated estimation algorithm, as well as limited to the space, 
we shall not discuss this derivational problem and shall leave it for our future work.

Note that in some papers like \cite{Hannan76, Hannan80} on identifying general ARMA or ARMAX models by maximum likelihood, instead of the covariance matrix $\Gamma$ used in our problem \eqref{eq:MLOriginalProblem}, an approximate covariance matrix based on $x(1), \cdots , x(N)$ with $N$ finite might be considered in some cases. 
We use $\Gamma$, since the likelihood function \eqref{eq:likelihood} can be accurately given directly based on the Gaussian assumption in this paper, with respected to the parameters in polynomial function matrices $A(z^{-1})$ and $B(z^{-1})$. And under a smaller $N$, using $\Gamma$ would have a better performance \cite{Hannan76}.

Denote by 
\begin{subequations}\label{ZmYl}
\begin{align}
  z(t) :=  \begin{bmatrix}
             Z_m(t)' & -Y_l(t)' 
           \end{bmatrix}',
\end{align}
with 
\begin{align}
    & Z_m(t)=  \begin{bmatrix}
                 \zeta_m(t-1)' & \cdots & \zeta_m(t-q)' \end{bmatrix}',\\
    & {Y}_l(t) = \begin{bmatrix}
        y_l(t)' & \cdots & y_l(t-r)' \end{bmatrix}'.
\end{align}
\end{subequations}
Then we have $x(t)= \zeta_m(t) + \Theta z(t)$ and the following results.

\begin{prop}[Jacobin and Hessian]\label{prop:Jacobin_Hessian}
The {\bf Jacobin matrix} of $J(\Theta)$ in \eqref{eq:J(A,B)} w.r.t. $\Theta$ is
\begin{align}\label{eq:Jacobin}
  \frac{\pdf J}{\pdf \Theta} = 2D\Theta'\Gamma^{-1} + \frac{2}{N}\sum_{t=1}^{N} (z(t)-D\Theta'\Gamma^{-1}x(t))x(t)'\Gamma^{-1},
\end{align} 
where
\begin{align}\label{eq:D}
  D= \diag \{ I_{mq},~0_{l(r+1)} \}.
\end{align}
Hence the {\bf gradient matrix} of $J(\Theta)$ is 
\begin{align}
    \nabla_\Theta {J} = \left( \frac{\pdf J}{\pdf \Theta} \right)'.
\end{align}
The {\bf Hessian matrix} of $J(\Theta)$ is 
\begin{subequations}\label{eq:HessianJAB}
\begin{align}
  \Hb(J) &= 2 \Hb_1  + \frac{2}{N}\sum_{t=1}^{N}\left( \Hb_2(t) + \Hb_3(t)\right),
\end{align}
where 
\begin{align} 
 \Hb_1 = -M_1D\otimes \Gamma^{-1}-K(\Gamma^{-1}\Theta D \otimes D\Theta'\Gamma^{-1} ), 
\end{align}
\begin{align}
\begin{split}
   \Hb_2(t) = K \Big(\Gamma^{-1}\Theta D\otimes & M_2(t)x(t)'\Gamma^{-1} \\
   & + \Gamma^{-1}x(t)M_2(t)'\otimes D\Theta'\Gamma^{-1}  \Big), 
\end{split}
\end{align}
\begin{align}
    \Hb_3(t) = \ M_2(t)M_2(t)' \otimes  \Gamma^{-1} + M_1D \otimes \Gamma^{-1}x(t)x(t)'\Gamma^{-1},
\end{align}
\end{subequations}
with
\begin{subequations}
\begin{align}\label{eq:M_1}
    & M_1 = D\Theta'\Gamma^{-1}\Theta-I, \\ 
   \label{eq:M_2}
    &M_2(t)=D\Theta'\Gamma^{-1} x(t) -z(t), \\
    &K:= K_{mq+rl+l,m} = \sum_{j=1}^{m}\left(\iota_j' \otimes I_{mq+rl+l} \otimes \iota_j \right),
\end{align}
$\iota_j$ denotes an $m$-dimensional column vector with the $j$-th entry equal to $1$ and other entries $0$. 
\end{subequations}
\end{prop}

\begin{proof}
From direct calculations, the differential of $J(\Theta)$ in \eqref{eq:J(A,B)} is
\begin{subequations}\label{eq:dJ}
\begin{align}
  \df J = \df J_1 + \df J_2,
\end{align}
where
\begin{align}\label{eq:dJ1}
    \df J_1 =& 2 \trace\left( D\Theta'\Gamma^{-1}\df \Theta\right), \\ \label{eq:dJ2}
    \df J_2 =& \frac{2}{N} \sum_{t=1}^{N} \trace \left( (z(t)-D\Theta'\Gamma^{-1}x(t))x(t)'\Gamma^{-1}\df \Theta \right),
\end{align}
\end{subequations}
$D$ is given in \eqref{eq:D}, with $\df J_1$, $\df J_2$ respectively corresponds to the differential of the two terms on the RHS of \eqref{eq:J(A,B)}. 
Hence from Lemma~\ref{lem:JacobinHessian}~(a) in Appendix~\ref{apdx:thms}, the Jacobin matrix \eqref{eq:Jacobin} is obtained.

Then we calculate the second-order differential of $J$. For $\df J_1$ in \eqref{eq:dJ1} we have
\begin{subequations}\nonumber
\begin{align}
    \dd J_1 := 2 \trace\left( d_1 \right),
\end{align}
\begin{align}
    d_1 &= \df \left(D\Theta'\Gamma^{-1}\df\Theta\right)
    = D(\df\Theta')\Gamma^{-1}\df\Theta - D\Theta'\Gamma^{-1}(\df\Gamma)\Gamma^{-1}\df\Theta \\
    &=D(I- \Theta'\Gamma^{-1}\Theta D)\df\Theta'\Gamma^{-1}\df\Theta  - D\Theta'\Gamma^{-1}\df\Theta D\Theta'\Gamma^{-1}\df\Theta\\
    &=-M_1D \df\Theta'\Gamma^{-1}\df\Theta  - D\Theta'\Gamma^{-1}\df\Theta D\Theta'\Gamma^{-1}\df\Theta.
\end{align}
\end{subequations}
For $\df J_2$ in \eqref{eq:dJ2} we have 
\begin{subequations}\nonumber
\begin{align}
\begin{split}
    \dd J_2 =& \frac{1}{N}\sum_{t=1}^{N}\trace\left(\dd( x(t)'\Gamma^{-1}x(t))\right) \\
    := &\frac{2}{N}\sum_{t=1}^{N}\left( \trace(d_2(t)) +  \trace(d_3(t))\right),
\end{split}
\end{align}
where
\begin{align}
    d_2(t) &=  2M_2(t) x(t)'\Gamma^{-1}\df \Theta D\Theta'\Gamma^{-1}\df \Theta,  \\
    d_3(t) &=  M_2(t)M_2(t)'\df\Theta'\Gamma^{-1}\df\Theta \\ 
    & \quad \quad \quad \quad +  M_1D\df\Theta'\Gamma^{-1}x(t)x(t)'\Gamma^{-1}\df\Theta,
\end{align}
\end{subequations}
with $M_1$, $M_2(t)$ given in \eqref{eq:M_1} \eqref{eq:M_2}.
Hence we have \eqref{eq:HessianJAB} based on Lemma~\ref{lem:JacobinHessian} in Appendix~\ref{apdx:thms} and the equation $\df(\trace(\cdot))=\trace(\df(\cdot))$.
\end{proof}

From Proposition~\ref{prop:Jacobin_Hessian}, the Hessian matrix of function $J(\Theta)$ is not a positive or negative definite matrix. This can be easily checked when processes $\zeta_m(t)$ and $y_l(t)$ are scalar. Hence problem \eqref{eq:MLOriginalProblem} is non-convex in general.
Furthermore, we give the following result on the stationary point of problem \eqref{eq:MLOriginalProblem}.

\begin{thm}\label{thm:equilibrium}
Point $\Theta_1$ satisfying \eqref{eq:Acons} is a stationary point of problem~\eqref{eq:MLOriginalProblem}, if and only if $\frac{\pdf J}{\pdf \Theta} |_{\Theta=\Theta_1} = 0$, i.e., 
\begin{align}\label{eq:equilibrium}
    N D\Theta_1' - D\Theta_1'\Gamma^{-1} \sum_{t=1}^{N}x(t)x(t)' + \sum_{t=1}^{N}z(t)x(t)'=0,
\end{align}
for $\Theta_1$. \eqref{eq:equilibrium} is also a necessary condition for $\Theta_1$ being a local minimum.
\end{thm}

\eqref{eq:equilibrium} in Theorem~\ref{thm:equilibrium} is a complex high order equation about parameter $\Theta$, and hence is difficult to theoretically analyze the extreme points of ARMAX low rank graphical estimation by maximum likelihood in \eqref{eq:MLOriginalProblem}, through calculating the analytic solution for stationary points.

Using optimization approaches like Newton's method with equality constraints \cite[Chapter 10]{ConvexBook04} based on Proposition~\ref{prop:Jacobin_Hessian} above and Proposition~\ref{prop:JH_cons} in Appendix~\ref{apdx:thms}, a suboptimal solution can be obtained for problem \eqref{eq:MLOriginalProblem} under constraints \eqref{eq:Acons}, \eqref{eq:Bcons}. To simply use Newton's method with linear equality constraints, $\vecc(\Theta)$ is used instead of matrix $\Theta$ in optimization.

Specifically for our problem, suppose Newton's method is chosen. 
For a vectorized $\vecc(\Theta)$, given an initial value $\vecc(\Theta_{\rm{in}})$ in the feasible domain satisfying all constraints and tolerance $\epsilon >0$.
The Newton step $\Delta \vecc(\Theta)$ is characterized by 
\begin{align}\label{eq:NetwonStep}
  \begin{bmatrix}
     \Hb(J) & C' \\
     C & 0
   \end{bmatrix}\begin{bmatrix}
                  \Delta \vecc(\Theta) \\
                  x
                \end{bmatrix}=\begin{bmatrix}
                                -\vecc(\nabla_{\Theta}J) \\
                                0
                              \end{bmatrix},
\end{align}
where we use the facts $\Hb(J)= \Hb_{\vecc(\Theta)}(J)$, $\nabla_{\vecc{\Theta}}J=\vecc(\nabla J)$ for any scalar function $J$,  $x$ is the associated optimal dual variable for the quadratic problem in Newton's method \cite[p.256]{ConvexBook04} , $\Hb(J), \nabla_{\Theta}J$ are given in Proposition~\ref{prop:Jacobin_Hessian}, and $C$ is given in Proposition~\ref{prop:JH_cons}.
The Newton decrement is given by 
\begin{align}\label{eq:NetwonDec}
    d(\vecc(\Theta))= \left( \Delta \vecc(\Theta)'\Hb(J(\Theta))\Delta \vecc(\Theta)  \right) ^{\frac{1}{2}}.
\end{align}

After obtaining an estimate $\hat{\Theta}$, the estimate of $\hat{y}_m(t)$ can be given by 
\begin{align}\label{eq:ymhat}
    \hat{y}_m(t)=\hat{H}(z)\hat{y}_l(t),
\end{align}
with $\hat{H}(z)$ calculated by \eqref{eq:H=A-1B}\eqref{eq:AcalBcalAzBz}\eqref{eq:ThetaAcalBcal}.

Now a realization procedure of our whole algorithm on estimating low rank graphical models can be given, which is shown in Algorithm~1 and abbreviated by LRG.  
The identification of the low rank latent variable $y_l$ for an AR graphical model and the estimation of $\hat{y}_l$ is realized in line $1$ to line $3$. The maximum likelihood estimation of the graphical ARMAX model is realized through a Newton's method with equality constraints in line $4$ to line $11$ in Algorithm~1. 
For backtracking line search in Algorithm~1, we recommend \cite{ConvexBook04} to readers for more details.
The optimization approaches to the two constrained optimization problems might be changed for a better performance in specific scenarios.


\begin{figure}
  \centering \hspace{-2.8mm}
\includegraphics[scale=0.54]{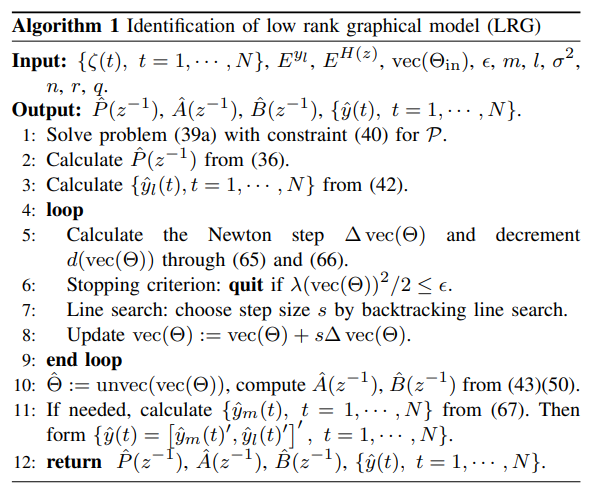}
\end{figure}

In this paper we suppose the dimension of the low rank latent variable $y_l(t)$, and the zero entries in $H(z)$ are known. 
The estimation of these information  will be discussed in our future work.
Next we shall theoretically analyze the convergence of the maximum likelihood approach for low rank ARMAX graphical estimation, when the data is enough. 

\subsection{Convergence analysis}\label{subsec:ARMAX_convergence}

Though the non-convexity makes our low rank ARMAX estimation by maximum likelihood difficult to obtain an optimal solution, a consistent estimate can be obtained for problem \eqref{eq:MLOriginalProblem} subject to \eqref{eq:Acons}\eqref{eq:Bcons} with enough sampling data. The statement will be proven in the following, based on some natural assumptions referring to \cite{Hannan76}.


\begin{assp}\label{assp1}
For the second-order stationary stochastic processes ${y}_l(t)$, $e_m(t)$ in low rank graphical ARMAX model \eqref{eq:ARMAX}, where we have assumed $e_m(t) \sim N(0, I_m)$, $e_m(t)$ and $y_l(t)$ independent, the following assumptions are needed to guarantee the consistency:
\begin{subequations}
\begin{align}
\label{eq:assp_e}
   \E\{e_m(t)e_m(s)'\}=\delta_{ts}I_m,\quad \delta_{ts}=\left\{\begin{array}{cc}
                                                                    1, & t=s, \\
                                                                    0 & {\rm otherwise};\end{array}\right. 
\end{align}
\begin{align}\begin{split}
  \Gamma_{y_l}(t) := & \lim_{N\rightarrow \infty}~\frac{1}{N}\sum^{N}_{n=1} y_l(n)y_l(n+t)' \\
  = & \int_{-\pi}^{\pi} e^{ij\omega}\Phi_{y_l} \df\omega, \quad \forall j \in \Zbb_{+}; 
\end{split}\end{align}
\begin{align}
  \E\{e_m(t) | \mathcal{F}_{n-1}(e_m) \}=0; 
\end{align}
\begin{align}\label{eq:assp_51}
\begin{split}
   \sigma_{ijk}(t):= & \E \{e_{(i)}(n)e_{(j)}(n)e_{(k)}(n-t) \} \\
  &{\rm ~exists,~for~} t> 0,~i, j, k\leq m; 
\end{split}\end{align}
\begin{align}\label{eq:assp_52}
\begin{split}
   \sigma_{ijkp}(t,s):= & \E \{e_{(i)}(n)e_{(j)}(n)e_{(k)}(n-t)e_{(p)}(n-s) \} \\
  &{\rm ~exists,~for~} t,s > 0,~i, j, k, p \leq m;
\end{split}\end{align}
\begin{align}\label{eq:assp_61}
  & \lim_{k\rightarrow \infty} \E \{e_m(n)e_m(n)' | \mathcal{F}_{n-k} \} = I_m, {\rm ~a.s.};\\ \label{eq:assp_62}
  & \lim_{p \rightarrow \infty}\E \{e_{(i)}(n)e_{(j)}(n)e_{(k)}(n-t) | \mathcal{F}_{n-p} \} = \sigma_{ijk}(t), {\rm ~a.s.}; 
\end{align}
\end{subequations}
where $\mathcal{F}_{n}$ denotes the $\sigma$-algebra generated by $e_m(s)$, for $s \leq n$.
\end{assp}

In abundance of literature on factor analysis model estimation, the process $e(t)$ or $e_m(t)$ is assumed to be ergodic. This is a stronger assumption than strictly stationary as in \eqref{eq:assp_e} and what we assume in this paper, which is not necessary for consistency. Assumption 
\begin{align}\label{eq:assp_7}
   \E \{e_m(n)e_m(n)' | \mathcal{F}_{n-1} \} = I_m,
\end{align}
is a sufficient condition for \eqref{eq:assp_51}, \eqref{eq:assp_52}, \eqref{eq:assp_61} and  \eqref{eq:assp_62}. Hence either pair of the above assumptions can be made to guarantee consistency.

Denote by $\hat{\varTheta} = (\hat{A}(z^{-1}), \hat{B}(z^{-1}))$, i.e., the estimate of polynomial matrices $(A(z^{-1}), B(z^{-1}))$ from \eqref{eq:MLOriginalProblem} under constraints \eqref{eq:Acons}\eqref{eq:Bcons}, and by $\varTheta_0$ the matrices $(A(z^{-1}), B(z^{-1}))$ with true parameters. Obviously $\varTheta_0$ satisfies the topology constraints. Then we have the following theorem on consistency.

\begin{thm}[Consistency of the maximum likelihood estimate] \label{thm:consistency}
Suppose Assumption~\ref{assp1} is satisfied and system~\eqref{eq:ARMAX} is identifiable. 
Given enough sampling data of $\zeta_m$ and real data of $y_l$,
a consistent estimate $\hat{\varTheta}$ can be obtained by using maximum likelihood on low rank graphical ARMAX estimation problem \eqref{eq:MLOriginalProblem}, under constraints \eqref{eq:Acons}\eqref{eq:Bcons} on model \eqref{eq:ARMAX}, i.e., $\hat{\varTheta} \rightarrow \varTheta_0$ a.s. (with probability one) for $N\rightarrow \infty$.
\end{thm}

\begin{proof}
Denote by ${\bf \Theta}(m_c)$ the analytic manifold constituted by the feasible space of polynomial matrices $(A, B)$ under the constraints, where $m_c$ is the McMillan degree of $A$. 
Based on the insightful construction details on \cite[p. 280]{Hannan80}, ${\bf \Theta}(m_c)$ may be covered by $\left(\begin{matrix}m_c+m-1 \\ m-1 \end{matrix}\right)$ coordinate neighborhoods. 
Construct an open neighborhood $\mathcal{N}_0$ of $\varTheta_0$ by taking one of the coordinate neighborhoods. 
Thus we may find a closed set, $L$ in $\mathcal{N}_0$, containing an open set containing $\varTheta_0$. 
Then the remainder of ${\bf \Theta}(m_c)$ can be covered by $\left(\begin{matrix}m_c+m-1 \\ m-1 \end{matrix}\right)-1$ open coordinate neighborhoods, denoted by $\mathcal{N}_k$ ($k \geq 1$), choosing these not to intersect with $L$.

To prove the consistency, we have to show that the absolute minimum overall ${\bf \Theta}(m_c)$ is attended in $\mathcal{N}_0$ for $N$ large enough. 

Denote by $\hat{\varTheta}(k)$ the optimising solution of problem \eqref{eq:MLOriginalProblem} in the region $\Ncal_k \cap {\bf \Theta}(m_c)$. For $\forall \varTheta \in \Ncal_k \cap {\bf \Theta}(m_c)$, it is obvious that for the greatest lower bound
\begin{align}\label{eq:11}
    \overline{\lim}_{N \rightarrow \infty}~J_N(\hat{\varTheta}(k)) \leq \overline{\lim}_{N \rightarrow \infty}~J_N(\varTheta),
\end{align}
where $\overline{\lim}$ denotes the superior limit. 
From \cite[Lemma 3]{Hannan76}, we have
\begin{align}\nonumber
\begin{split}
    \lim_{n\rightarrow \infty}  \frac{1}{N}  \sum_{t=1}^{N} & x(t,\Theta)'\Gamma^{-1}(\Theta)x(t, \Theta) =: J_2^\circ(\varTheta)\\
    & = \left<A(e^{-i\theta},\varTheta)^{-1}, A(e^{-i\theta},\varTheta_0) \right>,
\end{split}
\end{align}
and $\inf J_2^\circ(\varTheta)=J_2^\circ(\varTheta_0)$,
with $A(e^{-i\theta},\varTheta)$ denoting $A(z^{-1})$ with $z^{-1}=e^{-i\theta}$ and parameter values taken from $\varTheta$.
Hence 
\begin{align}\label{eq:12}
\begin{split}
   \inf_{\varTheta \in \Ncal_0}~\overline{\lim}_{N \rightarrow \infty} ~J_N(\varTheta) & = \log \det \Gamma(\varTheta_0) + {J_2^\circ}(\varTheta_0) \\
   & = \log \det \Gamma(\varTheta_0) + m.
\end{split}
\end{align}

Denote by $\lambda_{m}(\cdot)$ the smallest eigenvalue of a matrix, then we easily have $\lambda_{m}(\Gamma)>0$ from \eqref{eq:Gamma}.
Consider for any $\varTheta \in {\bf \Theta}(m_c)$,
\begin{align}\nonumber
    J_N(\varTheta)   \geq \log \det \Gamma(\varTheta) + \frac{1}{\gamma} \lambda_{m}(\Gamma(\varTheta))^{-1}\lambda_{m}(\mathcal{P}(\varTheta)),
\end{align}
where  
\begin{align}\nonumber
    \mathcal{P}(\varTheta)=\sum_{j=0}^{d}\sum_{k=0}^{d} U(j)\left[ N^{-1}\sum_{t=V+1}^{N}  x(t-j)x(t-k)' \right] U(k)', 
\end{align}
$\Gamma(\varTheta)$ is given in \eqref{eq:x(t)dist} by parameter matrices,
$\sum U(j)z^{-j} = {\rm adj}(A(z^{-1}))$ (i.e., the adjugate of $A$), $\gamma \geq |\det (A)|^2$, $d$ is the degree of a trigonometric polynomial $M_T(z^{-1})$, which satisfies $\varepsilon_M I_m \succeq M_T(z^{-1})- (A^{-1}BB^*A^{-*})^{-1} \succeq 0$, with $\varepsilon_M > 0$ a small constant.

Then from a tedious but direct procedure see such as in \cite[p. 285]{Hannan80}, it can be shown that there is $\varepsilon >0$, $N_0$ with $p(N_0 < \infty)=1$, such that for all $N \leq N_0$, $\lambda_{m}(\mathcal{P}(\varTheta))\geq \varepsilon$ for all $\varTheta$. Hence we have
\begin{align}\nonumber
  J_N(\varTheta) \geq m \log \lambda_{m}(\Gamma(\varTheta)) + {\varepsilon}{\gamma} ^{-1} \lambda_{m}(\Gamma(\varTheta))^{-1}.
\end{align}
By choosing suitable values of coefficients $\varepsilon$ and $\gamma$, the lower bound of $J_N$ satisfies
\begin{align}\label{eq:15}
  \underline{\lim}_{N \rightarrow \infty}~J_N(\hat{\varTheta}(k)) = \log \det \Gamma(\varTheta_0) + m,
\end{align}
only if $\hat{\varTheta}(k) \rightarrow \varTheta_0$. 
Note that for \eqref{eq:15}, we have to show that initial conditions can be ignored. This follows directly from \cite[p. 286]{Hannan80}.

From \eqref{eq:11}, \eqref{eq:12} and \eqref{eq:15}, we have 
\begin{align}
  {\lim}_{N \rightarrow \infty}~J_N(\hat{\varTheta}(k)) = J_N({\varTheta}_0),
\end{align}
leading to the consistency of maximum likelihood estimate on model \eqref{eq:ARMAX} in this paper.
\end{proof}

In the two-stage identification of the low rank process $y(t)$, which is sampled under noise,
we can only use filtered or estimated data $\hat{y}_l$ to solve this second graphical ARMAX estimation subproblem \eqref{eq:MLOriginalProblem}. 
From the proof of Theorem~\ref{thm:consistency}, it can be readily established that 
the proposed maximum likelihood approach using $\hat{y}_l$, still yields a unique convergent parameter estimate with enough data. 
As will be illustrated in Section~\ref{sec:example}, 
the $\hat{y}_l$ we designed in \eqref{eq:ylhat_new} brings the estimates in problem \eqref{eq:MLOriginalProblem} close to the real parameter values and real data, under an appropriate level of noise.

The above result can be verified from the aspect of equilibrium points. When $N\rightarrow \infty$, $\frac{1}{N}\sum_{t=1}^{N}x(t)x(t)'\rightarrow \Gamma$. 
Then problem \eqref{eq:J(A,B)} has only one equilibrium point from condition \eqref{eq:equilibrium},
\begin{align}
    \varTheta_0= -\left(\sum_{t=1}^{N}\zeta_m(t)z(t)' \right)\left(\sum_{t=1}^{N} z(t)z(t)' \right)^{-1}.
\end{align}
The satisfactory  convergence properties will also be demonstrated in Section~\ref{sec:example} by examples.

\section{Simulation examples}\label{sec:example}
\subsection{Example 1: a simulation example}
In this subsection, a simulation example will be presented on identifying a low rank vector process under measurement noise. 
The performance of Algorithm~1 will be evaluated. 
Furthermore, a comparative analysis will be provided between our approach considering measurement noise, and the traditional low rank (LR) identification approach presented in~\cite{CPLauto23}, which does not explicitly address the issue of noise.

A low rank stationary vector process $y(t)$ of dimension $7$, with a spectral density of rank $3$ is considered. 
The latent variable $y_l(t)=[y_{(4)}(t), y_{(5)}(t), y_{(6)}(t)]'$ is driven by 
$$
y_l(t)=W_l(z)\tilde{w}(t),
$$
with $\tilde{w}(t)$ a $3$-dimensional  i.i.d. random Gaussian process with variance $I$. $W_l$ is a sparse minimum-phase function matrix with
\begin{align}\nonumber
  & [W_l]_{11}=\frac{z^3}{z^3-0.8z^2-0.25z+0.2}, \quad [W_l]_{21}= \frac{z}{z^2-0.25}, \\ \nonumber
  & [W_l]_{22}= \frac{z^2}{z^2+0.3-0.1}, \quad [W_l]_{33}= \frac{z^2}{z^2-0.64},
\end{align}
and the other entries $0$.
The deterministic relation $H(z)$ is given by \eqref{eq:H(i)A(i)B(i)} with $A_{(i)}$, $B_{(i)}$ given in TABLE~\ref{tabel:H}. 
Based on Theorem~\ref{thm:identifiability}, the corresponding low rank ARMAX graphical model \eqref{eq:ARMAX} is identifiable.
Measurement noise $e(t)$ is of variance $\sigma^2I$ with $\sigma=0.1$.
Suppose the dependence topology inside the low rank latent variable (i.e., the zero entries of $\Phi_{y_l}^{-1}$) is known, and the topology inside the deterministic relation (i.e., the zero entries of $H(z)$) is given. 
%
%

\begin{table}[t]\centering
\caption{Value of $H(z)$ in Example 1.} \label{tabel:H}
\renewcommand\arraystretch{1.2}
\begin{tabular}{c|c|c}
    \toprule
       $i$       & $A_{(i)}$   & $B_{(i)}$\\ \hline \hline
      $1$ &  \scriptsize $z-0.9$   &  \scriptsize $[0.5z+0.15, ~0,~0]$ \\ \hline
      $2$ &  \scriptsize $z^2-0.64$         & \scriptsize $[z^2+0.3z-0.4 , ~z^2-0.3z-0.4,~0]$ \\ \hline
      $3$ &  \scriptsize $z^2-0.1z-0.42$  &  \scriptsize $[0, ~z^2-0.5z-0.14, ~z^2+z+0.24]$ \\ \hline
      $4$ &  \scriptsize $z^2+0.2z-0.48$   & \scriptsize \makecell{$[z^2-1.1z+0.3,  z^2-0.1z-0.3,$\\ $-2z^2-2z -0.32]$}\\
     \bottomrule
          \end{tabular}
\end{table}

In the overall example, the LRG estimation is realized by Algorithm~1, with Newton's method with equality constraints chosen for solving both optimization problem \eqref{eq:dualARgraph} and problem \eqref{eq:MLOriginalProblem} subject to \eqref{eq:Acons}\eqref{eq:Bcons}. 
The traditional LR estimation is realized by the procedure proposed in \cite{CPLauto23}, which 
estimates $y_l$ by regarding $\zeta_l(t)$ its observation neglecting the noise $e(t)$ in modeling, and then estimates $H(z)$ by least squares based on the data of  $\zeta_m$ and $\hat{y}_l$ from \eqref{eq:ylhat_old}. 

Given the measurement time series $\{\zeta(t),~\text{for~} t=1, \cdots, N\}$ with $N=1000$, LRG and LR identifications are run, respectively.
In the identification of the low rank latent variable, we use AR models with degree $n=5$ to estimate the dynamics of $y_l$.
The tolerance parameters for stopping the Newton's method optimizations are both set to be $\epsilon=10^{-12}$. The initial values of nonzero parameters in the low rank ARMAX graphical estimation are all set to be $1$.
In the identification of the internal deterministic relation $H(z)$, we suppose the degree of each entry in $A(z^{-1})$ and $B(z^{-1})$ is known, to better present the performance of the algorithms.

In an arbitrary run, parameters in models are identified, and the estimated time series $\hat{y}(t)$ are calculated. The estimations of low rank variable $y_l(t)$ and $y_m(t)$ are partly shown in Fig.~\ref{fig:y_hat}. In both the estimation of $y_l(t)$ and $y_m(t)$, algorithm LRG provides estimated time series highly close to the real data hidden behind the observed noise.
Because of overlooking the measurement noise in modeling, LR can only provide rough estimates of $y_l$ and $y_m$, which have approximate variation tendency with the real data, but inaccurate data estimations at each time step.
In several times of experiments, we find $\hat{y}_{(4)}(t)$ of LR diverges to values of a large order of magnitudes like $10^{21}$.
The reason is that the inaccuracy of $\hat{y}_l(t)$ caused by the neglect of $e_l(t)$ and $w(t)$ leads to expanded errors in identifying $H(z)$ without considering $e_m(t)$ specially.

\begin{figure*}[t]
\subfigure[The real values and estimations of $y_l(t)$.]{\hspace{-3mm}
        \includegraphics[scale=0.66]{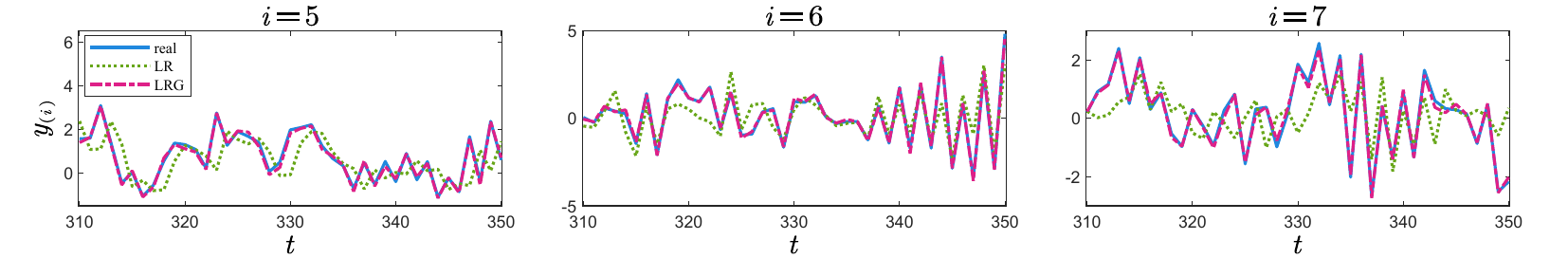}} \\ 
\subfigure[The real values and estimations of $y_m(t)$.]{ \hspace{-4.3mm}
        \includegraphics[scale=0.66]{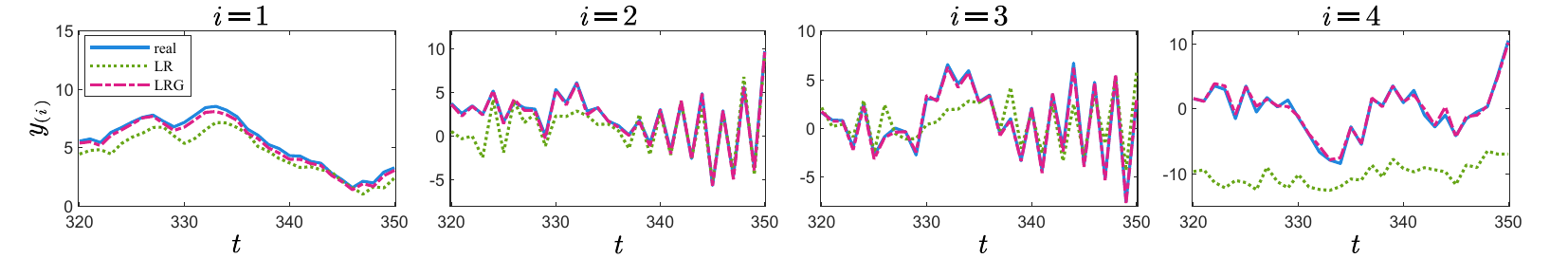} }
    \caption{Example 1: the estimation in one MC experiment of $y_l(t)=[y_{(5)}(t),~y_{(6)}(t),~y_{(7)}(t)]'$ in time period $t=310,\cdots, 350$, and of $y_m(t)=[y_{(1)},~y_{(2)},~y_{(3)},~y_{(4)}]'$ in time period $t=320,\cdots, 350$. The blue solid line denotes the real data of $y_{(i)}(t)$. The green dotted line and the pink dotted-dashed line denotes the estimated $y_{(i)}(t)$ by traditional low rank identification and the low rank graphical identification, respectively.}\label{fig:y_hat}
\end{figure*}

Therefore, even under a reasonable scale of measurement noise, LRG can provide a much superior estimation of the low rank time series $y(t)$ than LR. This result verifies the necessity of considering low rank graphical models with noise process $e(t)$ in identification issues under measurement noise, instead of identifying low rank processes directly. 
The results also highlight the potential of the algorithm LRG as a highly accurate filter for time series with internal correlations.

Next, with the above settings, we run $\text{MC}=20$ times Monte Carlo (MC) experiments to test the model estimation performance. 
For the AR (graphical) estimation of the low rank latent variable, since the minimum phase estimation of $P(z^{-1})^{-1}$ and hence $P(z^{-1})$ are not unique, we compare the (unique) estimated coherence of $y_l$ to examine.
Based on the average values of estimated parameters in MC experiments, the curves of estimated coherence between $y_{(k)}$ and $y_{(h)}$ are given in Fig.~\ref{fig:coh}
for ${S}_1:=\{(k,h) \in E^{y_l}, k\geq h \}$, $\theta\in [0, \pi]$.
The normalized average of integrated square error of coherence is calculated, to evaluate the error in low rank variable estimation, defined as
\begin{subequations}\label{eq:err_Phi}
\begin{align}
  \overline{err}_{\Phi}(y_l):= \frac{1}{ |{S}_1| \cdot \Vert \Phi_{y_l}(e^{i\theta}) \Vert}  \sum_{(k,h)\in {\bf \mathcal{S}}_1} err_{\Phi}(y_{(k)}, y_{(h)}),
\end{align}
with 
\begin{align}
    err_{\Phi}(y_{(k)}, y_{(h)}) := \int_{-\pi}^{\pi} \left( \big[\hat{\Phi}_{y_l}(e^{i\theta})-{\Phi}_{y_l}(e^{i\theta})\big]_{kh} \right)^2 \df \theta.
\end{align}
\end{subequations}

\begin{figure}
\hspace{-2.8mm}
\includegraphics[scale=0.6]{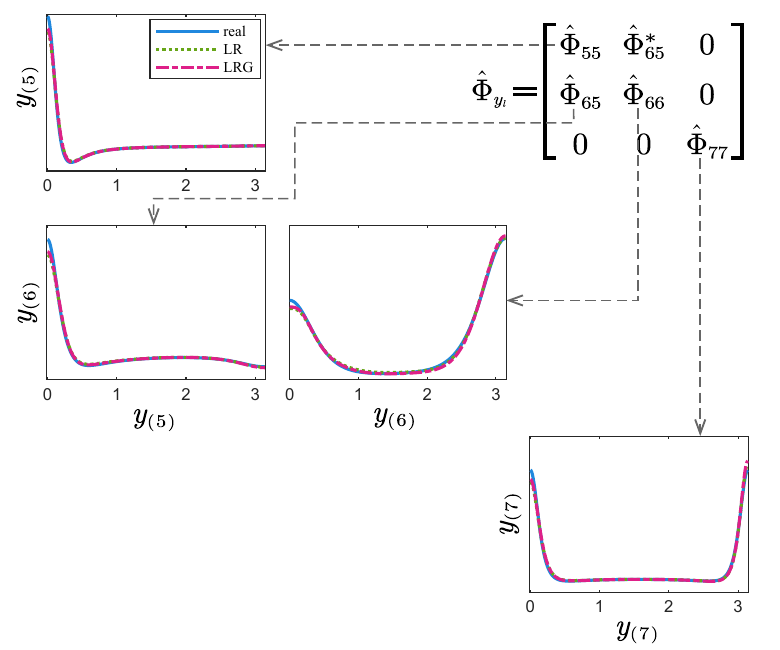}
\caption{Example 1: coherence of $y_l(t)$ calculated by real or estimated AR (graphical) models in MC experiments.}\label{fig:coh}
\end{figure}

In Fig.~\ref{fig:coh}, the coherence curves of estimated AR models by LR and LRG both fit closely the coherence curves of the real model of $y_l$.
The results show that both LR and LRG have good performance in identifying an innovation model for the low rank variable $y_l$ in $y$, under a reasonable scale of measurement noise. 
$\overline{err}_{\Phi}(y_l)$ of LR and LRG are shown in TABLE~\ref{tabel:crt}. The normalized errors of coherence estimation are both low, and LRG is a little bit better than LG.
However, as shown in Fig.~\ref{fig:y_hat} (a), LG can not provide good estimations of the time series by overlooking measurement noise from the modeling step, even though having nice performance in 
parameter identification.

\begin{table}[t]\centering
\caption{Quantitative evaluations of the estimates in Example 1} \label{tabel:crt}
\renewcommand\arraystretch{1.4}
\begin{tabular}{c|cc|cc}
    \toprule
       \multirow{2}*{Algo. }      & \multicolumn{2}{c|}{Estimates of parameters}   & \multicolumn{2}{c}{Estimates of $y(t)$}\\ \cline{2-5} 
            & $\overline{err}_{\Phi}(y_l)$ & ${fit}(\hat{H})$  & $\overline{fit}(\hat{y}_l)$ & $\overline{fit}(\hat{y}_m)$ \\ \hline \hline
         LR  & $1.13\times 10^{-2}$ & $-103.41$  & $50.96$ & $-3.00\times 10^{104}$  \\
         LRG & $5.41\times 10^{-3}$ & $98.22$  & $95.10$ &  $95.43$  \\
    \bottomrule
          \end{tabular}
\end{table}

For the estimation of $H(z)$, we use the fit of $\hat{\Theta}$ in low rank graphical ARMAX model \eqref{eq:ARMAX} for LRG, or $\hat{\Theta}$ composed of the parameters of  $\hat{A}(z^{-1})$, $\hat{B}(z^{-1})$ in an ARX model between $y_m$ and $y_l$ for traditional LR.
Define
\begin{align}\label{eq:fit_H}
 {fit}(\hat{H}) :=  100\left( 1- \frac{ \Vert \hat{\Theta}- {\Theta} \Vert_2}{\Vert {\Theta} \Vert_2} \right).
\end{align}

The calculated ${fit}(\hat{H})$ for LRG and LR are given in TABLE~\ref{tabel:crt}, from which we can see that the average of estimates in MC experiments of algorithm LRG provides a highly accurate estimate of function $H(z)$ with fit $98.22$, 
revealing that low rank graphical ARMAX identification can extract successfully the internal correlation inside a low rank process measured under noise.
In contrast, algorithm LR provides unsatisfactory estimates in identifying $H(z)$, implicating that identifying $H(z)$ based on \eqref{eq:H(z)ARMAX} directly does not work well when there are measurement noise.
In fact, for the two-stage approach LR, the estimation error of $\hat{y}_l$ is also a non-negligible factor inducing the poor estimate of $H(z)$.

To deeply evaluate the performance of low rank ARMAX graphical identification proposed in Section~\ref{sec:ARMAX} or line $4$-$10$ in Algorithm~1, the box plots of all parameter estimates in $H(z)$ in 20 MC experiments are drawn in Fig.~\ref{fig:Hbox} (a), where the real parameters are marked by green diamonds.
From Fig.~\ref{fig:Hbox} (a), all the estimates lie in a narrow range around the real parameters. The accuracy of a simpler model (like $H_1(z)$) and of models with more complex transfer functions (like $H_4(z)$) has no great difference.
Therefore, though the maximum likelihood estimation of low rank ARMAX graphical estimation is solved based on a nonconvex optimization problem \eqref{eq:MLOriginalProblem}, the algorithm converges to an estimate extremely close to real parameters with $N=1000$ data for $30$ unknown parameters. 
Note that, we tried several different initial points in some trials, in all the trails the low rank ARMAX graphical estimation converges to estimates close to the real parameters. 
The results are in accordance with Theorem~\ref{thm:consistency}.


\begin{figure*}[t]
\subfigure[Example 1: $N=1000, \sigma=0.1$.]{\hspace{-1.8mm}
        \includegraphics[scale=0.6]{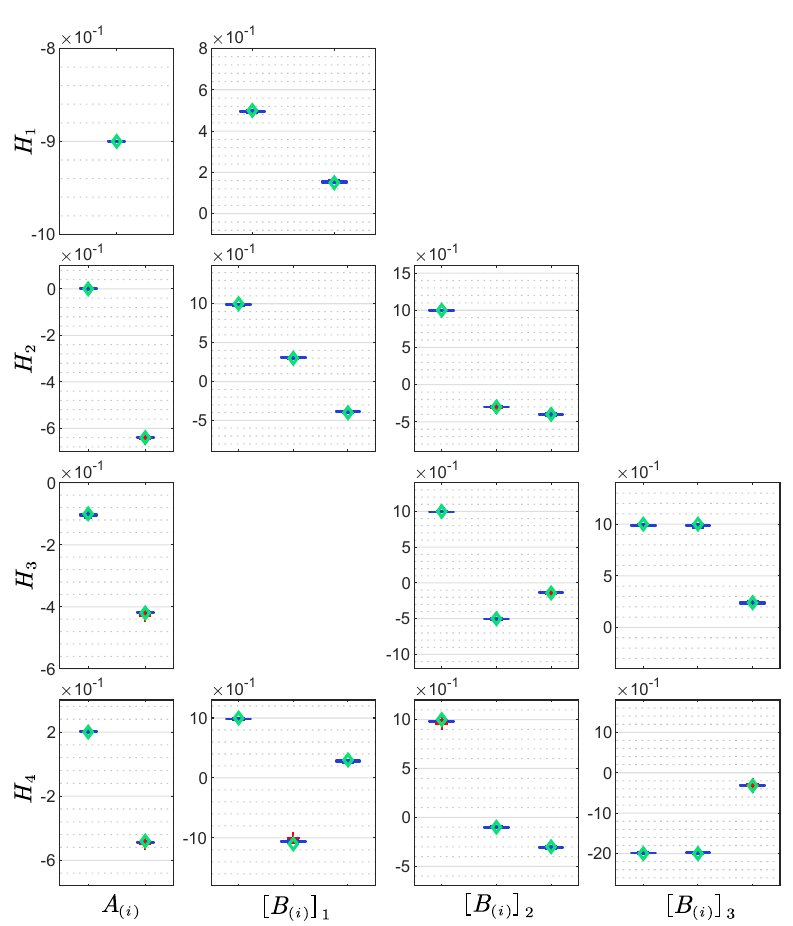}} 
\subfigure[Example 2: $N=1000, \sigma=0.9$.]{ \hspace{2mm}
        \includegraphics[scale=0.6]{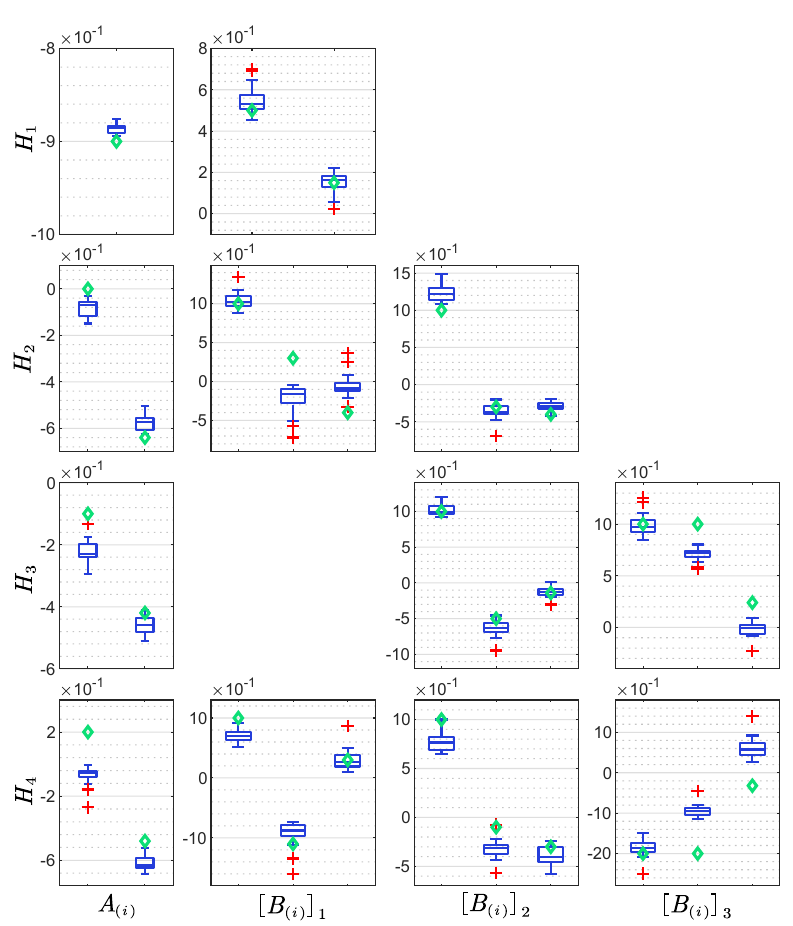} }
    \caption{Box plots of estimated parameters in $H(z)$ of Algorithm~1 in $20$ MC experiments, where the outliers are marked by red plus sign, and the real values of parameters given in TABLE~\ref{tabel:H} are marked by green diamonds.}\label{fig:Hbox}
\end{figure*}

Furthermore, the average fits of the estimation or filtering results of the subprocesses in `hidden' processes $y(t)$ for each MC trail are compared in TABLE~\ref{tabel:crt}.
Specifically, the average fit is defined as
\begin{subequations}\label{eq:fit_y}
\begin{align}
  \overline{fit}(\hat{y}_l) = \frac{1}{MC}\sum_{n_{MC}=1}^{MC} {fit}(\hat{y_l}, n_{MC}),
\end{align}
with
\begin{align}
  &{fit}(\hat{y_l}, n_{MC}) :=   100\left( 1- \frac{ \Vert \hat{\mathcal{Y}}_l- {\mathcal{Y}}_l \Vert_2}{\Vert {\mathcal{Y}}_l\Vert_2} \right),\\ 
  &{\mathcal{Y}}_l :=  \begin{bmatrix}
                         y_l(1), \cdots, y_l(N)
                       \end{bmatrix}',
\end{align}
\end{subequations}
and $\hat{\mathcal{Y}}_l$, $\overline{fit}(\hat{y}_m)$ defined similarly. 
Note that here the estimates of $\hat{y}$ are calculated based on the different parameter estimates in each MC trail rather than an average estimate after all the MC experiments. 
Along with Fig.~\ref{fig:Hbox} (a), the results shows that our algorithm LRG has reliable performance in one trail. However, the estimates of $y(t)$ by LR have low fits. Hence LR is not appropriate to be used as a filter for the denoising of low rank processes.

In summary, this simulation example provides an illustrative realization of our algorithm, demonstrating its effectiveness in estimating low rank graphical models.
The results confirm the superior performance of our algorithm in estimating a low rank process $y(t)$ under measurement noise, achieved by considering the low rank graphical model instead of the traditional low rank realization, and by designing and solving specialized problems.
In contrast, traditional LR identification performs surprisingly poorly, even yielding unacceptable estimates under noise, particularly when estimating the internal correlations and filtering the time series. 
Therefore, even a common level of measurement noise warrants special consideration in the identification and filtering of low rank processes.
The low rank graphical model \eqref{eq:zeta=yl+e} is a reliable and a superior choice on modeling, filtering and identification for low rank
processes under measurement noise.

\subsection{Example 2: sensitive analysis}
In this example we mainly evaluates the performance of the low rank graphical identification w.r.t. the level of measurement noise.
Maintain the system settings as they were in Example~1, change the data size $N$ and the scale coefficient $\sigma$ of the standard deviation of measurement noise $e(t)$.
The specific optimization approaches are chosen as the realization of low rank graphical identification in Algorithm~1.
In particular, our algorithm is realized under different levels of measurement noise in $[0.001, 1]$, and with $N$ chosen in set $\{250, 500, 1000\}$. 
For each value of $\sigma$ taken, $20$ MC experiments are performed, 
and the evaluation criteria $\overline{err}_\Phi (y_l)$, $fit(\hat{H})$ for parameter estimation, and $\overline{fit}(\hat{y}_l)$, $\overline{fit}(\hat{y}_m)$ for process estimation
are calculated based on \eqref{eq:err_Phi}\eqref{eq:fit_H}\eqref{eq:fit_y} accordingly.
The results are shown in Fig.~\ref{fig:ex2_cre} with a logarithmic horizontal axis. 

\begin{figure}
\hspace{-1.8mm}
\includegraphics[scale=0.66]{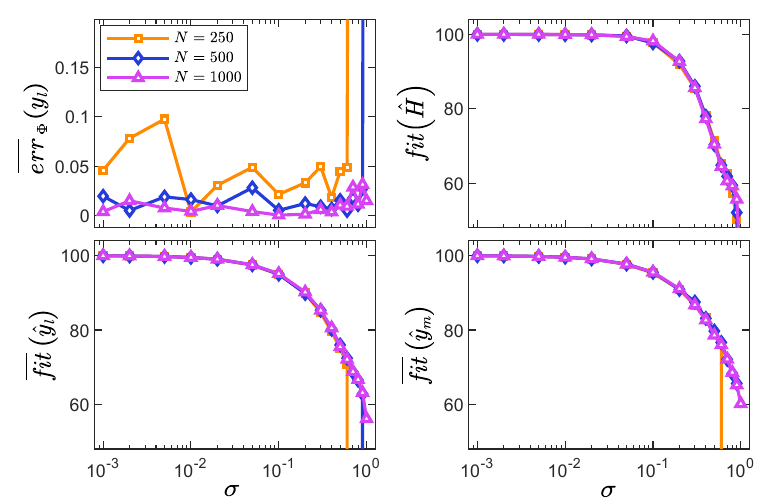}
\caption{Example 2: quantitative evaluation criteria of low rank graphical identification using different size of data $N=250, 500$ and $1000$ realized by Algorithm~1, with respect to $\sigma$. $\overline{err}_\Phi(y_l)$, $fit(\hat{H})$ are criteria for parameter estimation, $\overline{fit}(\hat{y}_l)$, $\overline{fit}(\hat{y}_m)$ are criteria for process estimation. }\label{fig:ex2_cre}
\end{figure}

From Fig.~\ref{fig:ex2_cre},  when $\sigma \leq 0.6$,
the filtering procedure and our ARMAX graphical estimation of parameters in $H(z)$ have similar performances for all the three $N$ we choose. 
The normalized average error of coherence estimation $\overline{err}_\Phi (y_l)$ of $y_l$ decreases as $N$ increases, and it remains below 0.04 for $\sigma \leq 0.5$ and $N=500, 1000$. 
When $\sigma >0.6$, estimations under a smaller size of data sharply deteriorate with the increase of $\sigma$.

For $N=1000$, our algorithm has a smoother performance w.r.t. $\sigma$. 
In this case, the coherence estimation of $y_l$ keeps having a low error less than 0.04. 
A steady good performance is gained on both the ARMAX graphical estimation and process filtering under a lower level of noise. 
For $\sigma\in [0.001, 0.2]$, the fits of $\hat{H}$, and of two estimated sub-processes $\hat{y}_l$ and $\hat{y}_m$ all keep greater than $90$.
For $0.2<\sigma\leq 0.6$, all the fits decrease prominently as $\sigma$ increases, and $\overline{err}_\Phi (y_l)$ has an increasing tendency, indicating a degradation in algorithm accuracy with increasing non-neglectable noise. 
For $\sigma >0.6$, all the fits decrease more sharply to a nearly unacceptable value less than 70. 

To check the convergent property of ARMAX graphical estimation under a large-level noise, the box plot of the parameter estimates with $N=1000, \sigma=0.9$ are drawn in Fig.~\ref{fig:Hbox} (b), with all outliers marked.
We can find that, though the estimates are biased from the real values when $\sigma$ increases, the estimates in different MC experiments have a fair convergent performance to a suboptimal point.

In summary, the low rank graphical identification algorithm we propose in this paper performs consistently well with enough data and under low noise levels, with accuracy gradually decreasing as the noise level increases.
When noise covariance gets closer to or exceeds that of the low rank process $y(t)$,  it is recommended to consider a more accurate model structure or measuring tool, such as using a higher-precision sensor or using multi-source information in real applications, or to employ identification methods specifically designed for handling large noise.

\section{Conclusion}\label{sec:conclu}

In conclusion, to solve the problem of identifying low rank processes under measurement noise, a two-stage low rank graphical identification algorithm is proposed. 
In the first stage, an innovation model for a subprocess of the low rank process is identified by maximum entropy covariance extension, and the estimates of the subprocess with respected to time are given.
Based on the estimated time series of the subprocess, a novel low rank ARMAX graphical estimation approach based on maximum likelihood is proposed, of which the identifiability and consistency with infinite data are proven.
Simulation examples demonstrate significantly improved performance against measurement noise compared with low rank identification without considering noise. 
The rank estimation, sparsity estimation and applications of low rank graphical identification will be studied in our future work.

\appendix

\section{Some useful results}\label{apdx:thms}
\begin{thm}[Identifiability of ARMAX model, \cite{Chen93}] \label{thm:chen93}
The following system described by ARMAX model 
\begin{subequations}\label{eq:genARMAX}
\begin{align}
&A(z^{-1})y(t)=B(z^{-1})u(t)+C(z^{-1})\omega(t), \quad  t \geq 0; \\
&y(t)=\omega(t)=0,~~u(t)=0, \quad  t\leq 0,
\end{align}
\end{subequations}
where $y(t), u(t)$ and $\omega(t)$ are $m$-output, $n$-input and $m$-driven noise, respectively; $A(z^{-1}), B(z^{-1})$ and $C(z^{-1})$ are given by the following equations 
\begin{subequations}\nonumber
\begin{align}
 & A(z^{-1}) = I +A_1z^{-1} + ... + A_q z^{-q},\quad q\geq 0,\\
 & B(z^{-1}) = B_1z^{-1} + ... +  B_r z^{-r},\quad r\geq 1, \\
 & C(z^{-1}) = I +C_1z^{-1} + ... + C_p z^{-p}, \quad p\geq 0,
\end{align}
\end{subequations}
is identifiable if and only if $A(z^{-1}), B(z^{-1})$ and $C(z^{-1})$ have no common left factor and $\rank\begin{bmatrix} A_q & B_r &C_p \end{bmatrix}=m.$
\end{thm}

Note that, in \cite{Chen93}, $z$ denotes the backward time shift operator, while the symbol $z$ in this paper (including the description of the above theorem) denotes forward time shift operator overall. 

The above identifiability theorem restricts the ARMAX model by a one-step time delay and the monic properties of both $A, C$. It is easy to extend the result to a more general ARMAX model, where only $A$ is restricted to be monic. Hence we give the following corollary. The proof is easy based on Theorem~\ref{thm:chen93} and hence omitted. 

\begin{cor}\label{coro:chen93}
System \eqref{eq:genARMAX}, where $y(t), u(t)$ and $\omega(t)$ are $m$-output, $n$-input and $m$-driven noise, respectively; $A(z^{-1}), B(z^{-1})$ and $C(z^{-1})$ are given by
\begin{subequations}\nonumber
\begin{align}
 & A(z^{-1}) = I +A_1z^{-1} + ... + A_q z^{-q},\quad q\geq 0,\\
 & B(z^{-1}) = B_0 + B_1z^{-1} + ... +  B_r z^{-r},\quad r\geq 0, \\
 & C(z^{-1}) = C_0 +C_1z^{-1} + ... + C_p z^{-p}, \quad p\geq 0,
\end{align}
\end{subequations}
is identifiable if and only if $A(z^{-1}), B(z^{-1})$ and $C(z^{-1})$ have no common left factor and $\rank\begin{bmatrix} A_q & B_r &C_p \end{bmatrix}=m.$
\end{cor}

\begin{lem}[\cite{Magnusbook}]\label{lem:JacobinHessian}
Given a scalar function $f(X)\in \Rbb$ of matrix which is differentiable at point $X \in \Rbb^{m\times n}$. Then,
\begin{itemize}
  \item[(a)] $\df f(X) = \trace ({\bf A} \df X) ~~\Leftrightarrow ~~ \frac{\pdf f}{\pdf X}={\bf A}.$
  \item[(b)] if $f(X)$ is second-order differentiable at point $X$. Suppose $\Bb, \Cb \in \Rbb^{n\times m}$ Then,
    \begin{align}\nonumber
        \df^2 f(X) & = \trace \left(\Vb (\df X) \Ub (\df X)'\right)\\ & ~~\Leftrightarrow ~~  \Hb(f(X))=\frac{1}{2}(\Ub' \otimes \Vb + \Ub \otimes \Vb')
    \end{align}
    or
    \begin{align}\nonumber
    \begin{split}
        \df^2 f(X) & = \trace \left(\Bb (\df X) \Cb \df X\right)\\
        & ~~\Leftrightarrow ~~  \Hb(f(X))=\frac{1}{2}K_{n,m}(\Cb' \otimes \Bb + \Bb' \otimes \Cb).
    \end{split}
    \end{align}
\end{itemize}
\end{lem}

\begin{prop}[Constraints w.r.t. $\vecc(\Theta)$]\label{prop:JH_cons}
The constraint~\eqref{eq:Acons} can be written as
\begin{subequations}\label{eq:consA2}
\begin{align}
\begin{split}
C_{A_{kj}}\vecc(\Theta)=0,  ~~&\text{for~} j=1,\cdots, m,\\
    &~~\text{for~}k=1,\cdots, q.
\end{split}
\end{align}
where
\begin{align}\label{eq:CA}
\begin{split}
C_{A_{kj}}= \Big[
           0 & _{(m-1)\times [(k-1)m^2+(j-1)m]}, ~~ I_{(-j)}' , \\
          & 0_{(m-1)\times [(q-k+1)m^2-jm]} ,~~ 0_{(m-1)\times mlr} \Big],
\end{split}
\end{align}
\begin{align}
    I_{(-j)}= \begin{bmatrix}
                \iota_1 & \cdots & \iota_{j-1} & \iota_{j+1} & \cdots & \iota_m
              \end{bmatrix},
\end{align}
$\iota_k$ denotes an $m$-dimensional column vector with the $k$-th entry equal to $1$ and other entries $0$ for $k=1, \cdots, j-1, j+1, \cdots, m$.
\end{subequations}

The constraint~\eqref{eq:Bcons} can be written as
\begin{subequations}\label{eq:consB2}
\begin{align}
\begin{split}
    c_{B_{jkh}}\vecc(\Theta)=0, ~~&\text{for~} j=0,1,\cdots, r, \\
                & ~~\text{for~} (k,h)\notin E^{H(z)},
\end{split}
\end{align}
where 
\begin{align}\label{eq:cB}
    c_{B_{jkh}} = \iota_{qm^2+jlm+(h-1)m+k,~qm^2+rlm}',
\end{align}
$\iota_{qm^2+jlm+(h-1)m+k, qm^2+rlm}$ denotes $(qm^2+rlm)$-dimensional vector with the $(qm^2+jlm+(h-1)m+k)$-th entry equal to 1 and others 0. 
\end{subequations}

Hence the linear constraint matrix $C$ in ARMAX low rank graphical estimation problem w.r.t. $\vecc(\Theta)$ can be written as a block column vector matrix composed of matrices in \eqref{eq:CA} and \eqref{eq:cB}, 
satisfying  $C\vecc(\Theta)=0$.

\end{prop}

\bibliography{reference}

\end{document}